\documentclass[aos]{imsart}

\RequirePackage[OT1]{fontenc}
\RequirePackage{amsthm,amsmath, amsfonts, amssymb}
\RequirePackage{mathtools}
\RequirePackage{natbib}
\RequirePackage[colorlinks,citecolor=blue,urlcolor=blue]{hyperref}
\RequirePackage{epsfig,longtable,xcolor,algpseudocode}
\usepackage[linesnumbered, ruled, vlined]{algorithm2e}
\usepackage{graphicx}
\usepackage{multirow}
\usepackage{pdflscape}
\usepackage{enumitem}
\usepackage{tabularx}
\usepackage{rotating}
\usepackage{bm}
\usepackage{xr}
\usepackage{tikz}
\usetikzlibrary{positioning}
\usetikzlibrary{decorations.pathreplacing,angles,quotes}
\usepackage{pgfplots}
\usepackage{subcaption}
\usepackage{makecell}
\usepackage{array}
\usepackage{booktabs}
\usepackage[font={footnotesize,it}]{caption}

\newcommand{\bfm}[1]{\ensuremath{\mathbf{#1}}}
\def\ba{\bfm a}     \def\bA{\bfm A}     \def\cA{{\cal  A}}     
\def\bb{\bfm b}          \def\cB{{\cal  B}}     
\def\bc{\bfm c}               
               
     \def\bE\\
{\bfm E}     \def\cE{{\cal  E}}     
         \def\cF{{\cal  F}}     
          \def\cG{{\cal  G}}     
          \def\cH{{\cal  H}}     
     \def\bI{\bfm I}          
          \def\cJ{{\cal  J}}     
               
          \def\cL{{\cal  L}}     
               
          \def\cN{{\cal  N}}     
               
     \def\bP{\bfm P}          
     \def\bQ{\bfm Q}          
               
          \def\cS{{\cal  S}}

\def\bv{\bfm v}          \def\cV{{\cal  V}}     
               
\def\bx{\bfm x}     \def\bX{\bfm X}     \def\cX{{\cal  X}}     
\def\by{\bfm y}     \def\bY{{\bfm Y}}          
\def\bz{\bfm z}     \def\bZ{\bfm Z}

\newcommand{\bfsym}[1]{\ensuremath{\boldsymbol{#1}}}
\def \balpha   {\bfsym{\alpha}}       \def \bbeta    {\bfsym{\beta}}
\def \bgamma   {\bfsym{\gamma}}       
\def \bepsilon {\bfsym{\epsilon}}     
         \def \btheta   {\bfsym{\theta}}

\newcommand{\bvarepsilon}{\bm{\varepsilon}}

\def \bTheta   {\bfsym{\Theta}}       
          
\def \bSigma   {\bfsym{\Sigma}}       
         
\def \bOmega   {{\bfsym{\Omega}}}

\DeclareMathOperator*{\argmin}{argmin}
\DeclareMathOperator*{\argmax}{argmax}

\def \RR	{\mathbb{R}}
\def \GG	{\mathbb{G}}
\def \PP {\mathbb{P}}

\def \mS {\mathcal{S}}
\def \mA {\mathcal{A}}

\def \mL {\mathcal{L}}

\def \mT {\mathcal{T}}
\def \mG {\mathcal{G}}
\def \mB {\mathcal{B}}

\def \wh{\widehat}
\def \wt{\widetilde}

\newcommand{\beq}  {\begin{equation}}
\newcommand{\eeq}  {\end{equation}}
\newcommand{\beqn} {\begin{eqnarray}}
\newcommand{\eeqn} {\end{eqnarray}}
\newcommand{\beqnn}{\begin{eqnarray*}}
\newcommand{\eeqnn}{\end{eqnarray*}}

\newcommand{\ltwonorm}[1]{\lVert#1\rVert_2}

\newcommand{\mbR}{\mathbb{R}}
\newcommand{\mbP}{\mathbb{P}}
\newcommand{\mbE}{\mathbb{E}}
\newcommand{\mcI}{\mathcal{I}}
\newcommand{\mcG}{\mathcal{G}}
\newcommand{\mcF}{\mathcal{F}}
\newcommand{\lof}{$L_0$-Fusion }

\startlocaldefs

\numberwithin{equation}{section}
\theoremstyle{plain}
\newtheorem{thm}{Theorem}[section]

\newtheorem{coro}{Corollary}[section]
\newtheorem{defn}[thm]{Definition}
\newtheorem{prop}{Proposition}[section]

\newtheorem{rem}{Remark}[section]

\newcommand{\forceindent}{\leavevmode{\parindent=1em\indent}}

\DeclarePairedDelimiter\paren{(}{)}
\DeclarePairedDelimiter\brak{[}{]}
\DeclarePairedDelimiter\brac{\lbrace}{\rbrace}
\DeclarePairedDelimiter{\floor}{\lfloor}{\rfloor}
\DeclarePairedDelimiter\norm{\lVert}{\rVert}
\DeclarePairedDelimiter\abs{\lvert}{\rvert}

\endlocaldefs

\begin{document}

\begin{frontmatter}
	\title{Supervised homogeneity fusion: A combinatorial approach}
	\runtitle{Homogeneity Fusion}
	%\thankstext{T1}{Footnote to the title with the ``thankstext'' command.}
	
\begin{aug}
		\author{\fnms{Wen} \snm{Wang}\thanksref{}\ead[label=e1]{wangwen@umich.edu}},
		\author{\fnms{Shihao} \snm{Wu}\thanksref{}\ead[label=e2]{wshihao@umich.edu}},
		\author{\fnms{Ziwei} \snm{Zhu}\thanksref{}\ead[label=e2]{ziweiz@umich.edu}},
		\author{\fnms{Ling} \snm{Zhou}\thanksref{}\ead[label=e2]{lingzhou@swufe.edu.cn}},
		\author{\fnms{Peter X.-K.} \snm{Song}\thanksref{}\ead[label=e2]{pxsong@umich.edu}}
\end{aug}

\begin{abstract}
	Fusing regression coefficients into homogenous groups can unveil those coefficients that share a common value within each group. Such groupwise homogeneity reduces the intrinsic dimension of the parameter space and unleashes sharper statistical accuracy. We propose and investigate a new combinatorial grouping approach called \lof  that is amenable to mixed integer optimization (MIO).
	On the statistical aspect, we identify a fundamental quantity called \emph{grouping sensitivity} that underpins the difficulty of recovering the true groups. We show that $L_0$-Fusion achieves grouping consistency under the weakest possible requirement of the grouping sensitivity: if this requirement is violated, then the minimax risk of group misspecification will fail to converge to zero. Moreover, we show that in the high-dimensional regime, one can apply $L_0$-Fusion coupled with a sure screening set of features without any essential loss of statistical efficiency, while reducing the computational cost substantially. On the algorithmic aspect, we provide a MIO formulation for $L_0$-Fusion along with a warm start strategy. 
% 	and analyze a new algorithm via modern optimization technique and segment neighborhood method to explore homogeneity. 
% 	Theoretical analysis indicates that the proposed method reconstructs the oracle estimator, that is, the unbiased least-square estimator given the true grouping, leading to consistent reconstruction of grouping structures and informative features, as well as to optimal parameter estimation. 
	Simulation and real data analysis demonstrate that \lof 
% 	combines the benefit of grouping pursuit with that of feature selection, and
	exhibits superiority over its competitors in terms of grouping accuracy. 	
\end{abstract}

\end{frontmatter}

\section{Introduction} 
\label{sec:intro}
Identifying homogeneous groups of regression coefficients has received increasing attention because the resulting regression model provides better scientific interpretations and enhance predictive performance in many applications. In some occasions, features or covariates naturally act in groups to influence outcomes, so knowing group structures of the features help scientists gain new knowledge about a physical system of interest. From a modeling perspective, aggregating covariates with similar effects along with the response reduces model complexity and improves interpretability, especially in the high-dimensional regime. There have been a flurry of works under this direction; see for example \cite{bondell2008simultaneous,shen2010grouping,zhu2013simultaneous,ke2015homogeneity,jeon2017homogeneity}, among others. There is a vast literature  in discovering homogeneous groups of observations or individuals in overly heterogeneous population. However, these existing methods cannot be applied to our problem that aims to group regression parameters. Identifying group structures of regression parameters is crucial to learn the underlying heterogeneous covariates' effects, which is then leveraged to reach a more appropriate model for data analyses. A partial list of the literature includes  \cite{ke2016structure,shen2015inference,ma2017concave,lian2017homogeneity}, just name a few. The focus of this paper is on pursuing homogeneous groups of regression coefficients in which we do not have any prior knowledge about their true group structures.  %We focus on pursuing homogeneous groups of features. 

Homogeneity fusion is carried out routinely in environmental health sciences in a manual and subjective manner  
% forming exposure mixtures of toxic agents such as endocrine disrupting compounds (EDCs) (e.g. PBA, phthalates and heavy metals) remains an open problem.  
to evaluate the effect of a given set of toxicants on certain health outcomes.  Consider $p$ toxicants, whose concentrations are denoted by $X_1,\ldots, X_p$ respectively, $q$ other covariates $\{Z_k\}_{k = 1} ^ q$ and a outcome variable $Y$.  Scientists typically consider a linear regression model $Y \sim \sum_{j=1}^p \beta_j X_j + \sum_{k = 1} ^ q \alpha_k Z_k$ to evaluate effect of a mixture $A := \sum_{j=1}^p \beta_j X_j$ on outcome $Y$. 
% Principal component analysis (PCA) has been the default approach to estimate $\{\beta_j\}_{j \in [p]}$ for long. Nevertheless, it has never gained popularity in environmental health sciences: PC-type mixture has both positive and negative loading coefficients, which are determined in an unsupervised learning fashion with no use of outcome $Y$. Thus, such PC-type  mixtures often lack meaningful scientific interpretations. Alternatively, 
One common practice to reduce model complexity and facilitate scientific interpretation is aggregating the exposure of similar toxicants to yield a sum-mixture (e.g. $\sum_{j} X_j$). For example, SumDEHP is a sum of four phthalates, MECPP, MEOHP, MEHHP, and MEHP, which quantifies total DEHP exposure from products such as PVC plastics used in food processing/packaging materials as well as building materials and medical devices \citep{Schettler2006, Kobrosly2012, Braun2014}. See also \cite{Marsee2006, Marie2015} for another sum-mixture called SumAA that adds three extra phthalates MBP, MiBP, and MBzP to SumDEHP. Learning such a sum-mixture structure requires the toxicants within the same mixture to share the same regression coefficients in the linear model. Unfortunately, in practice the formation of a sum-mixture is done manually by scientists in an \textit{ad hoc} fashion.  There has been long of interest to develop a data-driven homogeneity fusion methodology that provides a needed statistical toolbox for scientists to identify and include important toxicants, while excluding unimportant ones,  in the formation of a toxic mixture. This new approach can greatly reduce subjectivity in data processing and yield robust scientific conclusions and insights on the relationship between toxicants and outcome. This motivates us to pursue parsimony by regularizing coefficients $\beta_j$ in addition to homogeneity pursuit of those nonzero coefficients in our methodology. 
% \zzw{Where is the fusion part? It seems that only feature selection is needed. Song: Fusion occurs when all nonzero $\beta_j$ take the same parameter value.}
% It is quite common that only a subset of the toxicants are present in $A$. 
% The new statistical method developed in this paper will be applicable not only in our motivating example but in many other practical studies to answer important scientific questions, which otherwise cannot be answered with existing methods.
% The salient feature of grouping pursuit is simultaneous estimation of grouping and sparse structures. 

Suppose that the true linear model with $K_0$ groups of non-zero coefficients takes the form:  
\begin{eqnarray}
\label{model}
&&Y = \sum_{j=1}^p \beta  ^ *_j X_j + \sum_{k = 1} ^ q \alpha^*_k Z_k  +\varepsilon, \hspace{0.5cm}\beta  ^ *_j \in \{0, \gamma^*_1, \gamma^*_2, \dots, \gamma^*_{K_0}\}, ~\forall j \in [p], 
\end{eqnarray}
where $\varepsilon \sim \cN(0, \sigma ^ 2)$, and where the coefficients $\{\beta  ^ *_j\}_{j = 1} ^ p$ belong to a set including 0 and $K_0$ unknown different nonzero values $\{\gamma^*_k\}_{k = 1} ^ {K_0}$. Note that the group membership of each nonzero $\beta_j$ is not observed in data collection.  Write $\balpha ^ * = (\alpha^*_1, \ldots, \alpha^*_q)^\top \in \RR ^ q$, $\bbeta ^ *= (\beta  ^ *_1, \ldots, \beta  ^ *_p)^\top \in \RR ^ p$ and $\bgamma ^ * = (\gamma^*_1, \ldots, \gamma ^ *_{K_0})^\top \in \RR ^{K_0}$. Our main goal in this paper is to estimate $\bgamma^*$, $\bbeta^*$ and $\balpha^*$ simultaneously based on an independent and identically distributed ($i.i.d.$) sample $\{(\bx_i, \bz_i, y_i)\}_{i = 1} ^ n$ of size $n$. In the case of high dimension, $\bbeta$ is often assumed sparse so that we perform feature selection and grouping simultaneously to ensure statistical consistency. 

We now review and discuss some important works related to model \eqref{model}. \cite{shen2010grouping} considered model \eqref{model} without $\{Z_k\}_{k \in [q]}$ and proposed to minimize the following objective with respect to $\bbeta$: 
$
    S_1(\bbeta) = n^{-1}\sum_{i=1}^{n}(y_i - \sum_{j=1}^px_{ij}\beta_j)^2 + \lambda_1\sum_{j < j'}J_{\tau}(|\beta_j - \beta_{j'}|), 
$
where $\lambda_1$ is a tuning parameter that is associated with fusion strength, and $J_{\tau}(z) = \min(z\tau^{-1}, 1)$ is a surrogate of the indicator function $1_{z \neq 0}(z)$, with $\tau > 0$ representing the approximation error of $J_{\tau}(z)$ to the $L_0$ penalty $1_{z \neq 0}(z)$. Such penalty on the pairwise difference can lead to redundant comparisons and extra computational complexity. Note that there is no sparsity regularization in $S_1(\bbeta)$. As an extension, \cite{zhu2013simultaneous} considered simultaneous grouping pursuit and feature selection by further penalizing individual coefficients, that is, minimizing
$
    S_2(\bbeta) = n^{-1}\sum_{i=1}^{n}(y_i - \sum_{j=1}^px_{ij}\beta_j)^2 + \lambda_1\sum_{(j, j') \in \mathcal{E}}J_{\tau}\big(\big||\beta_j| - |\beta_{j'}|\big|\big) + \lambda_2 \sum_{j=1}^pJ_{\tau}(|\beta_j|). 
$ 
Here $\mathcal{E}$ is the edge set of an undirected graph with $p$ nodes representing $\{X_j\}_{j = 1}^p$. If $X_i$ and $X_j$ can be grouped, then there is an edge between nodes $i$ and $j$; otherwise, there is no edge. Available prior knowledge of $\mathcal{E}$ reduces computational burden and improves estimation efficiency. 
% However, as will shown later in the numerical studies, inappropriate knowledge of the $\mathcal{E}$-net usually causes biased estimation and then leads to wrong group structures. \zzw{We did not investigate this in the simulation, did we? Song: I agree that we did not look into this issue. We simply delete this statement.} \wsh{We used the order in OLS coefficients to determine $\cE$ in the simulation. It is data-driven, not known a prior.}
However, it is always challenging in practice to obtain a plausible estimate of $\mathcal{E}$, which makes the method less appealing. \cite{ke2015homogeneity} proposed a different method named as clustering algorithm in regression via data-driven segmentation (CARDS). They use a preliminary estimate to determine ``adjacent" coefficient pairs for fusion and only penalize distances between the two coefficients in each adjacent pairs by folded concave penalty function. Therefore, the CARDS estimator depends on the initial ordering of the coefficients, which could be unstable especially when the effect sizes are small (e.g. weak signals).

We propose to pursue homogeneity and sparsity simultaneously through a combinatorial approach called $L_0$-Fusion. Specifically, we estimate $\bbeta ^ *$ by the least squares with an exact group constraint ($\beta ^ *_j$ can only take $K_0$ distinct nonzero values) and an $L_0$ sparsity constraint. To obtain this estimator, we formulate the corresponding optimization problem as a mixed integer optimization (MIO) problem. \cite{bertsimas2016best} demonstrated that MIO provides a computationally tractable approach to solve the classical best subset selection (BSS) problem of a practical scale: With the sample size in thousands and the dimension in hundreds, a MIO algorithm can achieve provable optimality in minutes. Such success of MIO and the similar combinatorial nature of the BSS problem inspire us to seek for a MIO formulation of the \lof problem. Our main contributions are summarized as follows: (a) To the best of our knowledge, it is the first time that we formulate the group pursuing as a MIO problem; (b) we show that the estimator derived from the \lof problem achieves grouping consistency once the loss function is reasonably sensitive to a certain grouping error; (c) we discover that the grouping sensitivity requirement in (b) turns out to be necessary (up to a universal constant) for any approach to achieve selection and grouping consistency; 
% \cite{zhu2013simultaneous} also considered the theoretical investigation of constrained $L_0$-version with strong conditions on the prior knowledge $\mathcal{E}$-net. 
(d) we provide a warm start algorithm with convergence guarantee for the \lof problem, in order to accelerate the MIO solver. 

The rest of the article is organized as follows. Section \ref{sec:stat} introduces the \lof method and a ``screen then group'' strategy to tackle high dimension. It also presents all the statistical theory, including the selection and grouping consistency of the \lof method and a necessary condition to achieve such consistency. Section \ref{sec:mixed} introduces our MIO formulation for the \lof problem together with a warm up algorithm. Section \ref{sec:num} demonstrates significant superiority of the \lof approach over existing ones in terms of grouping accuracy in both low-dimensional and high-dimensional regimes. We also apply \lof to a metabolomics dataset to aggregate concentration of similar lipids to predict the body mass index (BMI). The appendix includes all technical details, including the proofs of major theoretical results.

\section{Statistical methodology and theory}\label{sec:stat}

\subsection{Notation}

We use  regular letters, bold regular letters and bold capital letters to denote scalars, vectors and  matrices  respectively. 
For any positive integer $n$, we denote $\{1, \ldots, n\}$ by $[n]$. For any two sets $\mA$ and $\mB$, let $\mA \backslash \mB := \mA \cap \mB ^ c$. For  any  vector $\ba$ and  matrix $\bA$,  we  use $\ba^\top$ and $\bA^{\top}$ to  denote  the  transpose  of $\ba$ and $\bA$ respectively.
Given $\mB = \{i_1, \ldots, i_{|\cB|}\} \subset [p]$, we use $\bX_{\mB}$ to denote the submatrix of $\bX$ with columns indexed in $\cB$ and use  $\bbeta_{\mB}$ to denote $(\beta_{i_1}, \dots, \beta_{i_{|\cB|}})^\top$.
Given  any $a,b\in\RR$,  we  say $a\lesssim b$ if  there  exists  a  universal constant $C >0$ such that $a\le Cb$; we say $a \gtrsim b$ if there exists a universal constant $c > 0$ such that $a \ge  cb$; we say $a \asymp b$ if $a \lesssim b$ and $a \gtrsim b$. For any event $\cA$, we use $I(\cA)$ to denote the indicator function associated with $\cA$, i.e., $I(\cA) = 1$ if $\cA$ occurs, and $I(\cA) = 0$ otherwise. 

\subsection{\lof with feature screening}\label{sec:ultra_met}

Suppose we have $n$ independent observations $(\bx_i, \bz_i, y_i)_{i \in [n]}$ from model \eqref{model}. Our paper revolves around the following combinatorial optimization problem to achieve feature selection and homogeneity fusion simultaneously: 
\begin{eqnarray}
\label{eq:prob}
&&\min_{\balpha \in \RR ^ q, \bbeta \in \RR ^ p, \bgamma \in \RR ^ K}\sum_{i=1}^n(y_i - \bx_i ^ \top \bbeta - \bz_i^{\top}\balpha)^2, \\
%\label{eq:con1}
&& \text{subject to: } \hspace{0.5cm} \beta_j \in \{0, \gamma_1, \gamma_2, \dots, \gamma_K\}  , \forall j \in[p], \nonumber\\
%\label{eq:con2}
&&\hspace{2.1cm} \|\bbeta\|_0 := \sum_{j=1}^pI(\beta_j \neq 0) \leq s. \nonumber
\end{eqnarray} 
The first constraint requires the non-zero group number to be bounded by $K$, and the second constraint requires the sparsity of $\bbeta$ to be bounded by $s$. Given that the problem above restricts the $\ell_0$-norm of $\bbeta$ and also fuses the components of $\bbeta$, we refer to it as the \lof problem. Without the grouping constraint, \eqref{eq:prob} boils down the well-known best subset selection (BSS) problem \citep{Gar65, HLe67, BKM67} with subset size $s$. Note that problem \eqref{eq:prob} is NP-hard because of the cardinality and grouping constraint. Despite of the computational challenge, Section \ref{sec:mio_l0} provides a MIO formulation of \eqref{eq:prob} that is amenable to modern integer optimization solvers such as GUROBI and MOSEK. In our numerical study, when the dimension $p \le 100$, GUROBI can solve the \lof problem within seconds.

However, under practical setups, $p$ is often in thousands or even millions. Directly solving the \lof problem under such a scale is computationally burdensome or even prohibitive. To tackle this, we propose a ``\textit{screen then group}'' strategy. In the screening stage, let $\wt{\cS}$ denote a screening set generated by a preliminary feature screening procedure, the examples of which include, but are not limited to,  penalized least squares methods \citep{tibshirani1996regression,fan2001variable,  zhang2010nearly}, sure independence screening \citep{fan2008sure} or sparsity constraint method \citep{needell2009cosamp,fan2020best}. Suppose $\wt{\cS}$ enjoys the sure screening property, i.e., the true support set $\cS ^ 0\subseteq \wt{\cS}$ with high probability. Then in the \textit{grouping} stage,   
we perform \lof on the reduced design $\bX_{\wt{\cS}}$ to generate groups of nonzero coefficients, so that we work with lower-dimensional version of problem \eqref{eq:prob}.

\begin{center}
        \begin{algorithm}[t]
            \caption{CoSaMP($\bX, \by, \wh\bbeta_0, \pi, l, \tau$)}\label{alg_iht}
            \KwIn{Design matrix $\bX$, response $\by$, initial value $\wh\bbeta_0$, projection size $\pi$, expansion size $l$, convergence threshold $\tau>0$}
            %Input: Design matrix $\bX$, response $\by$, initial value $\wh\bbeta_0$, projection size $\pi$, expansion size $l$, convergence threshold $\tau>0$.
            \begin{algorithmic}[1]
                \State ~~ $t$ $\gets$ $0$
                \State ~~ \textbf{repeat}
                \State ~~~~ \hspace*{0.02in} $\mG_t \gets \mT_{\text{abs}}(\nabla\mL(\wh\bbeta_t),l)$
                \State ~~~~ \hspace*{0.02in} $\mS_t^{\dagger}\gets$supp$(\wh\bbeta_t)\cup\mG_t$
                \State ~~~~ \hspace*{0.02in} $\wh\bbeta^{\dagger}_t\gets(\bX_{\mS_t^\dagger}^{\top}\bX_{\mS_t^\dagger})^{+}\bX_{\mS_t^\dagger}^{\top}\by$
                \State ~~~~ \hspace*{0.02in} $\mS_t\gets\mT_{\text{abs}}(\wh\bbeta_t^{\dagger},\pi)$
                \State ~~~~ \hspace*{0.02in} $\wh\bbeta_{t+1}\gets (\bX_{\mS_t}^{\top}\bX_{\mS_t})^{+}\bX_{\mS_t}^{\top}\by$
                \State ~~~~ \hspace*{0.02in} $t\gets t+1$
                \State ~~ \textbf{until}
                $\|\wh\bbeta_t-\wh\bbeta_{t-1}\|_2<\tau$
                \State~~ $\wh\bbeta^{\text{cs}}\gets\wh\bbeta_t$            
         \end{algorithmic}
         \KwOut{$\wh\bbeta^{\mathrm{cs}}$ }
        \end{algorithm}
\end{center}

We choose CoSaMP (Compressive Sampling Matching Pursuit), an iterative two-stage hard thresholding algorithm proposed by \cite{needell2009cosamp}, as our variable screener.  Algorithm \ref{alg_iht} presents its pseudocode. CoSaMP performs two rounds of hard thresholding in each iteration: it first expands the model by recruiting the largest coordinates of the gradient (lines $3$-$4$) and then contracts the model by discarding the smallest components of the refitted signal on the expanded model (lines $5$-$7$). \cite{fan2020best} showed that under a high-dimensional sparse regression setup, CoSaMP (referred to as IHT therein) can achieve sure screening properties within few iterations under highly correlated designs. In addition, \cite{zhu2021early} showed numerically that CoSaMP yields much fewer false discoveries than LASSO, SCAD and MCP on early solution paths, particularly in the presence of high correlations among predictors. These supporting results signify CoSaMP as an efficient and reliable screener that can help substantially reduce the dimension while retaining the true signals. We emphasize that the low false discovery rate (FDR) here is crucial to controlling the dimension of the reduced design on which the \lof procedure becomes computationally tractable. %can finish within reasonable amount of time. 

\subsection{Statistical theory}\label{sec:stat_theory}

In this section, we prove that the global minimizers of problem~\eqref{eq:prob} reconstruct the ideal ``oracle estimator'', i.e., the estimator with prior knowledge of the true grouping, under a ``degree-of-separation'' condition. To understand how the proposed method performs under high dimensions, in the following we derive necessary and sufficient conditions to achieve grouping consistency as well as  selection consistency. Define the parameter space $\bTheta(K, s)  := \{\btheta = (\bbeta^{\top}, \balpha^{\top})^{\top} \in \RR ^ {p + q} \,|\, \bgamma \in \RR ^ K, \beta_j \in \{0, \gamma_1, \dots, \gamma_K\}, \forall j \in [p],  \|\bbeta\|_0 \leq s\}.$ Denote the index operator for the elements of $\bbeta$ with value $r$ by $\mG(\bbeta; r) = \{j\in[p]\,|\, \beta_j = r\}$ and the grouping operator by $\mathbb{G}(\bbeta) = \{\mG(\bbeta; r)\,|\, r\neq 0, \mG(\bbeta;r)\neq\emptyset\}$. Let $|\mG(\bbeta; r)|$ and $|\mathbb{G}(\bbeta)|$ be the cardinality of $\mG(\bbeta; r)$ and $\mathbb{G}(\bbeta)$, respectively.

Throughout this section, we write the $n \times p$ design matrix $\bX = (\bX_1, \dots, \bX_p)$ and the $n \times q$ matrix $\bZ = (\bZ_1, \dots, \bZ_q)$, where $\bX_j$ and $\bZ_k$ are the $j$th and $k$th columns of $\bX$ and $\bZ$, respectively.  

\subsubsection{Sensitivity to grouping accuracy}
We first define a distance $d(\bbeta,\bbeta')$ between two groupings that correspond to $\bbeta$ and $\bbeta'$ respectively:

\begin{defn}[Distance between groupings]
	Let $\mcF(\bbeta, \bbeta'):=  
	\brac{f \text{ is injective}:\mathbb{G}(\bbeta)\to \mathbb{G}(\bbeta')}.$
	Then for any $\bbeta, \bbeta'$ such that $|\mathbb{G}(\bbeta)| \leq |\mathbb{G}(\bbeta')|$, define 
	\begin{eqnarray}
	    \label{eq:group}
		d(\bbeta, \bbeta') := \min_{f \in \mcF(\bbeta, \bbeta')} \bigg|\bigcup_{\mcG_2 \in \mathbb{G}(\bbeta')} \bigg\{\mcG_2 \big\backslash \bigcup_{\mcG_1 \in \mathbb{G}(\bbeta)} \{\mcG_1 \cap f(\mcG_1)\}\bigg\}\bigg|. 
	\end{eqnarray}  
	\label{def:sensitivity}
\end{defn}

\noindent This distance is the minimum number of grouping labels that need to be changed to match $\mathbb{G}(\bbeta)$ and $\mathbb{G}(\bbeta')$. Specifically, $\cup_{\cG_1 \in \mathbb{G}(\bbeta)} \{\cG_1 \cap f(\cG_1)\}$ collects all the variables that are consistently labeled by $\mathbb{G}(\bbeta)$ and $\mathbb{G}(\bbeta')$ based on a mapping $f$. Therefore, $\cup_{\mcG_2 \in \mathbb{G}(\bbeta')} \{\mcG_2 \big\backslash \bigcup_{\mcG_1 \in \mathbb{G}(\bbeta)} \{\mcG_1 \cap f(\mcG_1)\}\}$ means to capture all the variables with inconsistent group labels in $\mathbb{G}(\bbeta)$ and $\mathbb{G}(\bbeta')$ based on mapping $f$. Figure \ref{fig:degree_sep} illustrates a specific setup with two possible grouping maps $f ^ {(1)}$ and $f ^ {(2)}$ in $\cF(\bbeta, \bbeta')$. One can see that $f ^ {(1)}$ gives three inconsistent group labels (three crosses) between $\mathbb{G}(\bbeta)$ and $\mathbb{G}(\bbeta')$, while $f ^ {(2)}$ gives four. Therefore, $f ^ {(1)}$ minimizes the objective in \eqref{eq:group}, implying that $d(\bbeta, \bbeta') = 3$.

\begin{figure}[h]
    \begin{tabular}{cc}
        \includegraphics[width=.28\linewidth]{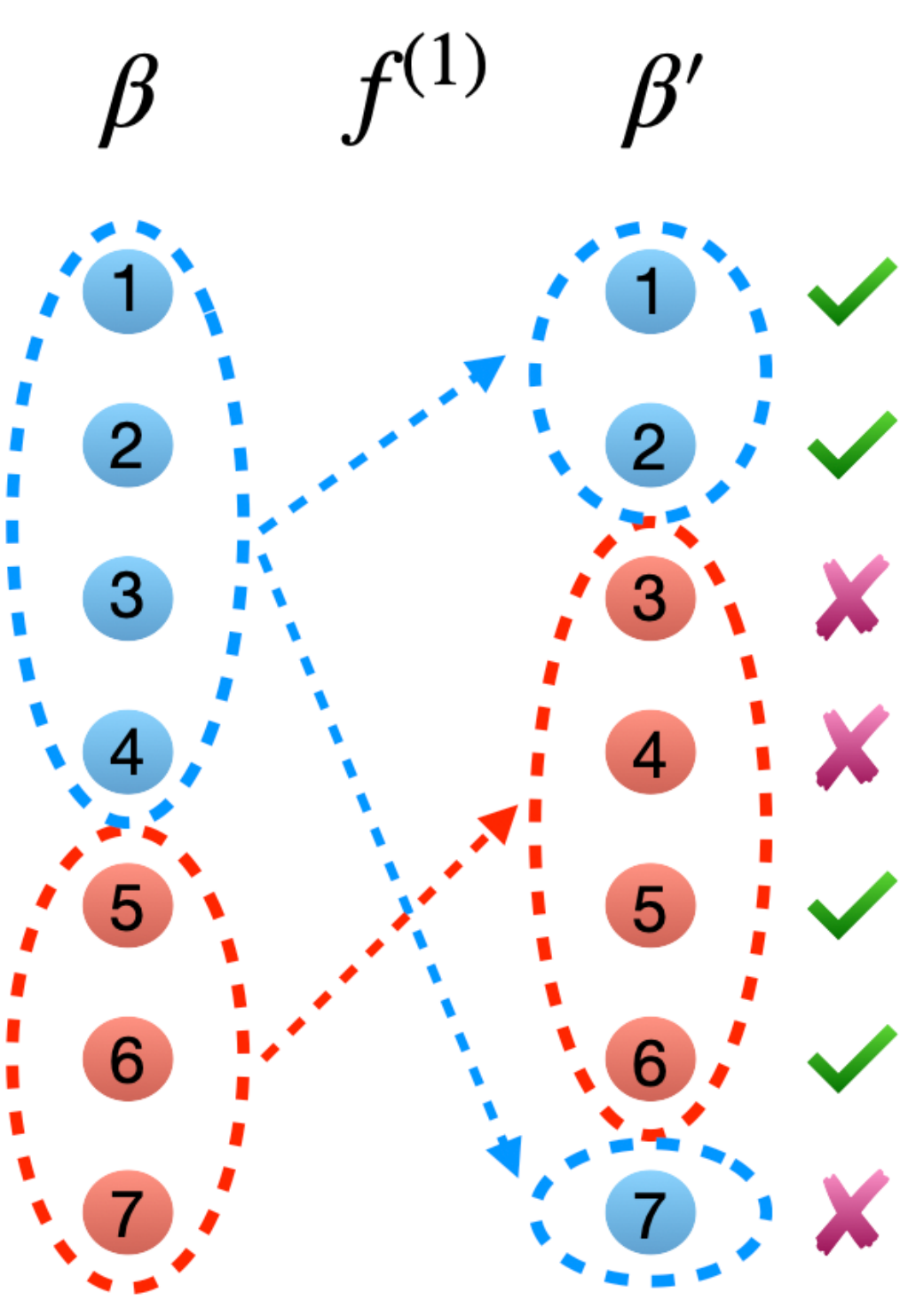} & \hspace{2cm}\includegraphics[width=.28\linewidth]{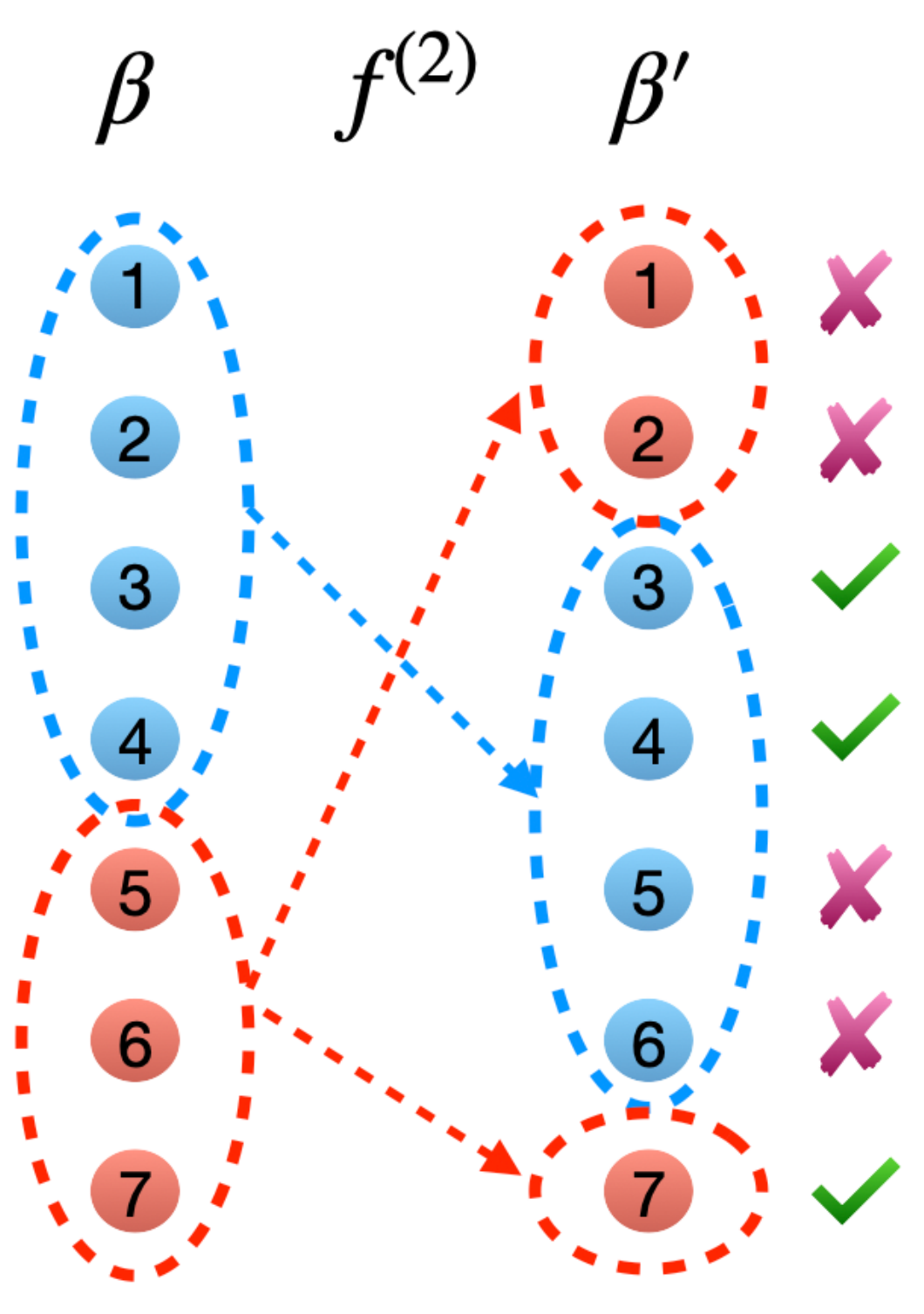}
    \end{tabular}
    \caption{Illustration of the grouping maps and grouping distance. Here $\bbeta, \bbeta' \in \RR ^ 7$, $|\mathbb{G}(\bbeta)| = |\mathbb{G}(\bbeta')| = 2$ and $|\cF(\bbeta, \bbeta')| = 2$. We write $\cF(\bbeta, \bbeta') = \{f^ {(1)}, f ^ {(2)}\}$ and illustrate these two maps on the left and right panels respectively. For clarity, we let $\cG$ and $f(\cG)$ share the same color for any $\cG \in \mathbb{G}(\bbeta)$. On the right border of each panel, for each feature, we use a check (cross) to indicate the consistency (inconsistency) between $\mathbb{G}(\bbeta)$ and $\mathbb{G}(\bbeta')$ according to the given grouping map.} 
    \label{fig:degree_sep}
\end{figure}

Next, we define a sensitivity measure of mean squared error (MSE) with respect to grouping error, which is shown later to determine the difficulty of identifying the true grouping:
\begin{defn}[Grouping sensitivity]
	\begin{align}
		c_{\min} \equiv c_{\min}(\btheta^*, \bX, \bZ) = \min_{\substack{{\btheta \in \bTheta(\abs{\mathbb{G}(\bbeta^*)}, \norm{\bbeta^*}_0)} \\ {\mathbb{G}(\bbeta) \neq \mathbb{G}(\bbeta ^ *)}}}\frac{\|\bX (\bbeta - \bbeta ^ *) + \bZ (\balpha - \balpha^*)\|^2_2}{n \max\left(d(\bbeta, \bbeta ^ *), 1 \right)}, \label{eq:dos}
	\end{align}
	where $\btheta^*=(\bbeta^{*\top},\balpha^{*\top})^\top\in\RR^{p+q}$.
\end{defn}

\noindent In words, $c_{\min}$ is the minimum increase of MSE due to a falsely grouped variable. A small $c_{\min}$ suggests that the MSE is insensitive to false grouping and thus makes it difficult to identify the true grouping. 
% {\color{red} add an example to justify that $c_{\min}$ is well defined. }

\subsubsection{Sufficient condition}
Given a grouping status $\mathbb{G}(\bbeta)$, define
\[
	\bX_{\mathbb{G}(\bbeta)} := \bigg(\sum_{k \in \mG(\bbeta; \gamma_1)}\bX_k, \dots, \sum_{k \in \mG(\bbeta; \gamma_{\abs{\mathbb{G}(\bbeta)}})}\bX_k\bigg), 
\]
which is a groupwise collapsed matrix by summing up columns of $\bX$ according to the groups in $\mathbb{G} (\bbeta)$. 

\begin{defn}[Oracle least squares estimator]
	Given the true coefficient $\bbeta ^ *$, the oracle least squares estimator $\hat{\btheta}^{\mathrm{ol}} = (\hat{\bbeta}^{\mathrm{ol} \top}, \hat{\balpha}^{\mathrm{ol}\top})^{\top}$ is defined as  
	\[
		\hat\btheta ^ {\mathrm{ol}} := \argmin_{\btheta: \mathbb{G}(\bbeta)=\mathbb{G}(\bbeta ^ *)} \|{\bY-\bX\bbeta-\bZ\balpha}\|_2 ^ 2.
	\] 
	More specifically, in $\hat{\bbeta}^{\mathrm{ol}} = (\hat{\beta}^{\mathrm{ol}}_1, \dots, \hat{\beta}^{\mathrm{ol}}_p)^{\top}$,  
	$\hat{\beta}^{\mathrm{ol}}_j$ is $\hat{\gamma}_k$ if $j \in \mathbb{G}(\bbeta ^ *; \gamma^*_{ k})$; $k = 1, \dots, K_0$, and $\hat{\beta}^{\mathrm{ol}}_j$ is $0$ if $j \in \mathbb{G}(\bbeta ^ *; 0)$, where 
	\[
		(\hat{\bgamma}^\top, \hat{\balpha}^\top) = (\hat{\gamma}_1, \dots, \hat{\gamma}_{K_0}, \hat{\balpha}^\top) = \argmin_{(\bgamma ^ \top, \balpha^{\top})^{\top} \in \mathbb{R}^{K_0 + q}}\|{\bY - \bX_{\mathbb{G}(\bbeta ^ *)}\bgamma - \bZ\balpha}\|^2_2.
	\]
\end{defn}

For any estimator $\hat{\btheta}=(\hat{\bbeta}^\top, \hat{\balpha}^{\top})^{\top}$ of $\btheta ^ *$, define the $0$-$1$ grouping risk $\cL_g(\hat{\btheta}; \btheta ^ *) := \mbP\paren[\big]{\mathbb{G}(\hat\bbeta)\neq \mathbb{G}(\bbeta ^ *)}$. Denote the solution to the \lof problem~\eqref{eq:prob} by $\hat{\btheta} ^ \mathrm{g} = (\hat{\bbeta}^{\mathrm{g}\top},\hat{\balpha}^{\mathrm{g}\top})^{\top}$. 
% We now derive a nonasymptotic probability error bound for simultaneous grouping pursuit and feature selection, based on which we prove the oracle estimator. 
Recall that $K_0=\abs{\mathbb{G}(\bbeta^*)}$ and further define $s_0=\norm{\bbeta^*}_0$. The next theorem says that $\hat\btheta ^ {\mathrm{g}}$ consistently recovers $\hat\btheta ^ {\mathrm{ol}}$ when the grouping sensitivity $c_{\min} \gtrsim \log(pK_0) / n$. Section \ref{sec:nece} shows that this lower bound is necessary to achieve grouping consistency. 
% Without loss of generality, assume that a global minimizer of~\eqref{eq:prob} exists.  

%\begin{thm}
%	\label{thm:bound}
%If $K = K_0$, $s \geq s_0$, $c_{\min} \geq d_1\sigma^2\frac{\log(pK_0)}{n} + d_2\sigma^2\frac{K_0}{n}$ for some constant $d_1 > 18$ and $d_2 > 12$, we have 
%\begin{itemize}
%	\item[(a)] $P(\hat{\btheta} \neq \hat{\btheta}^{\mathrm{ol}}) \leq 4\exp\left[-\frac{n}{18\sigma^2}\left( c_{\min} - 18\sigma^2\frac{\log(pK_0)}{n} - 12\sigma^2\frac{K_0}{n}
%	\right)\right]$; 
%	\item[(b)] $\hat{\btheta}$ consistently reconstructs $\hat{\btheta}^{\mathrm{ol}}$, that is,
%	$P\left(\mathbb{G}_{K_0}(\hat{\bbeta}) \neq \mathbb{G}_{K_0}(\hat{\bbeta}^{\mathrm{ol}}), \hat{\balpha} \neq \hat{\balpha}^{\mathrm{ol}}\right) \to 0$ as $n, p, K_0 \to \infty$.
%	\item[(c)] $n^{-1}E(\|\bX(\hat{\bbeta} - \bbeta_0) + \bZ(\hat{\balpha} - \balpha_0)\|^2_2) = (1 + o(1))n^{-1}
%	E\left(\|\bX(\hat{\bbeta}^{\mathrm{ol}} - \bbeta_0) + \bZ(\hat{\balpha}^{\mathrm{ol}} - \balpha_0)\|^2_2\right)$.
%\end{itemize}
%\end{thm}
%

\begin{thm}
	\label{thm:suff}
	Suppose that $K=K_0$ and $s=s_0$ in \eqref{eq:prob}.
	%		there exists a universal constant $C_1 > 18$ such that 
	We have
	\begin{align*}
		\mbP(\hat\btheta ^ {\mathrm{g}}\neq \hat\btheta^{\mathrm{ol}})\leq
		 6\exp\brak[\bigg]{-\frac{3n}{40\sigma^2}\brac[\bigg]{c_{\min}-\frac{\sigma^2}{n}  \paren[\big]{27\log(pK_0)+12}}},
	\end{align*}
    which implies that when $c_{\min} \ge \frac{\sigma^2}{n}  \{d_1\log(pK_0)+ 12\}$ for some universal constant $d_1 > 27$, $\hat\btheta^{\mathrm{g}}$ consistently reconstructs $\hat\btheta^{\mathrm{ol}}$, i.e., as $n,p \to\infty$, $\cL_g(\hat \btheta^{\mathrm{g}}; \btheta ^ *)\to 0$. 	
	%				$n^{-1}\E\bigl\Vert \bX(\hat\btheta-\btheta_0) \bigr\Vert^2\leq (1+o(1))n^{-1}E\left\Vert \bX(\hat\btheta_0-\btheta_0)\right\Vert^2=(1+o(1))\frac{K_0\sigma^2}{n}$.
\end{thm}

% Theorem~\ref{thm:suff} says that $\hat{\btheta}$ consistently reconstructs the oracle estimator $\hat{\btheta}^{\mathrm{ol}}$ as long as the degree-of-separation condition is satisfied, which is,
% \begin{eqnarray}
% 	\label{eq:cmin_suf}
% 	c_{\min} \geq d_3\frac{\log(p K_0)+\frac{12}{27}}{n}\sigma^2
% 	,
% \end{eqnarray} 
% where $d_3 > 27$ is a constant. 
% The lower bound of $c_{\min}$ in necessary condition~\eqref{eq:cmin_nece} and that in sufficient condition~\eqref{eq:cmin_suf} they are at the same order. 
% {\color{red}  Corollary for high-dimension.  }

Define the sure screening event as $\cE := \{\cS^0 \subseteq \wt{\cS}\}$. Let $\hat{\btheta}^{\mathrm{sg}}$ denote the solution under the ``\textit{screen then group}'' strategy. The following corollary says that as long as $\cE$ enjoys sure screening with high probability, $\hat\btheta ^ {\mathrm{sg}}$ also consistently recovers $\hat\btheta ^ {\mathrm{ol}}$. Many variable screening techniques provably yield such a sure screening set under reasonable assumptions on the signal and design, e.g., Sure Independence Screening \citep{fan2008sure}, LASSO \citep[][Theorem~7.21]{wainwright2019high}, CoSaMP \citep[][Theorem~3.1]{fan2020best}, etc. In the subsequent numerical study, we choose CoSaMP to pre-screen variables for $L_0$-Fusion given its robustness against design collinearity. 

\begin{coro}\label{cor:high-d}
    When $K = K_0$, $s=s_0$ and $c_{\min} > \frac{\sigma^2}{n}  \paren[\Big]{27\log(pK_0)+12}$, 
	%		and $c_{\min}\geq C_1\sigma^2\frac{\log\{p(K_0-1)\}}{n}$, 
	we have 
	\[
	    \mbP(\hat{\btheta}^{\mathrm{sg}} = \hat\btheta^{\mathrm{ol}}) \ge 1 - 6\exp\brak[\bigg]{-\frac{3n}{40\sigma^2}\brac[\bigg]{c_{\min}-\frac{\sigma^2}{n}  \paren[\Big]{27\log(pK_0)+12}}} - \mbP(\cE^c).
	\]
\end{coro}

\subsubsection{Necessary condition}\label{sec:nece}
For $\ell >0$, consider the following subspace of $\bTheta(K_0, s_0)$:
\[
    \bTheta_c(K_0, s_0, \ell) := \left\{\btheta: \btheta \in \bTheta(K_0, s_0), c_{\min}(\btheta, \bX, \bZ) \geq \ell \right\}. 
\]
We now present a lower bound for the minimax $0$-$1$ grouping risk over $\bTheta_c(K_0, s_0, \ell)$, which enables us to deduce the necessity of the lower bound of $c_{\min}$ in Theorem \ref{thm:suff} (up to a universal constant) in connection to a theoretical guarantee for the selection and grouping consistency simultaneously. For notational convenience, define the following subspace $\wt\bTheta(K_0, s_0)$ of $\bTheta(K_0, s_0)$ with well separated signal strengths across groups and balanced group sizes: 
\[
\begin{aligned}
\wt \bTheta & (K_0, s_0) := \\
& \big\{\btheta ^ * \in \bTheta(K_0, s_0):  {|\beta^*_j-\beta^*_{j'}|\geq 1, \forall \beta^*_j \neq \beta^*_{j'} }, |\mG(\bbeta ^ *, \gamma_k ^ *)| \le 2|\mG(\bbeta ^ *, \gamma_{k'} ^ *)|, \forall k, k' \in [K_0] \big \}. 
\end{aligned}
\]
% Note that by Theorem 1 of \cite{shen2013constrained}, feature selection alone requires that 
% \[
%     c_{\min} \geq d_1 \frac{\log p}{n} \sigma^2,
% \]'
% for some positive constant $d_1 \leq 1/4$ that may depend on $\bX$. In words, the grouping sensitivity must be well bounded from below to correctly identify the true model, which can be further translated into an upper bound of order $\exp(nc_{\min}/(d_1\sigma^2))$ on $p$.
% As pointed out in \cite{zhu2013simultaneous}, the problem of recovering the oracle estimator in the sense of simultaneous grouping pursuit and feature selection is more difficult than that of feature selection alone. 
% To derive a lower bound requirement for $c_{\min}$, we construct an approximate least favorable situation under $P$,

% to avoid super-efficiency.  

	%	Then $\inf_{\hat\bbeta}\sup_{\btheta_0\in B_0(K,s,\ell)}\mbP\paren[\big]{\mathbb{G}(\hat\bbeta)\neq\mathbb{G}(\bbeta ^ 0)}\to 0, \text{ as }n,p\to\infty$ implies $l\geq \frac{\sigma^2\paren[\big]{\log(Kp)-\log 2}}{2nr(\bX,\bZ,s)}$. 

\begin{thm}
	\label{the:nece}
% 	{\color{red} Remark on $r$}
    Define $$r(\bX, \bZ, K_0, s_0) := \frac{\max_{1 \leq j \leq p}n^{-1}\|\bX_j\|^2_2} {\min_{\substack{{\btheta^* \in \wt\bTheta(K_0, s_0)}}}c_{\min}(\btheta^*, \bX, \bZ)}.$$
	For any $K_0\geq 1$, $p\geq s_0\geq K_0$ and $\ell>0$, we have
    \begin{align*}
		\inf_{\hat\btheta}\sup_{\btheta^*\in \bTheta_c(K_0, s_0, \ell)} \cL_g(\hat{\btheta};\btheta^*)\geq 1-\frac{2nr(\bX,\bZ,K_0, s_0)\ell+\sigma^2\log 2}{\sigma^2\log(\floor{\frac{K_0+3}{4}}p)}.
    \end{align*}	
    Consequently, if $\inf_{\hat\btheta}\sup_{\btheta^* \in \bTheta_c(K_0, s_0, \ell)}\cL_g(\hat{\btheta};\btheta^*) \to 0,\text{as } n, p \to \infty$, one must have that $$\ell \geq \frac{\sigma^2\big\{\log\bigl(\floor{\frac{K_0+3}{4}}p\bigr)-\log 2\big\}}{2nr(\bX, \bZ, K_0, s_0)}.$$ 
\end{thm}

Quantity $r(\bX, \bZ, K_0 ,s_0)$ plays an important role in the lower bound above, which deserves some discussion.  We conjecture that under an restricted eigenvalue (RE) assumption \citep{BRT09, VBu09, NRW12} and an assumption of bounded marginal variance of the features, $r(\bX, \bZ, K_0 ,s_0) \lesssim 1$. Specifically, write $\wt{\bX} = (\bX, \bZ)$. Under the RE condition that $n^ {-1} \|\wt{\bX} \bv\|_2 ^ 2 \ge \kappa \ltwonorm{\bv} ^ 2$ for any $\bv \in \RR ^ {p + q}$ with $\|\bv\|_0 \le (2s_0 + q)$ and some $\kappa > 0$, we have that 
\[
    \min_{\substack{{\btheta^* \in \wt\bTheta(K_0, s_0)}}} c_{\min}(\btheta ^ *, \bX, \bZ) \ge \min_{\substack{{\btheta^* \in \wt\bTheta(K_0, s_0)}}} \min_{\substack{{\btheta \in \bTheta(|\mathbb{G}(\bbeta ^ *)|, \|\bbeta ^ *\|_0)} \\ {\mathbb{G}(\bbeta) \neq \mathbb{G}(\bbeta ^ *)}}}\frac{\kappa \|\bbeta - \bbeta ^ *\|_2 ^ 2}{\max(d(\bbeta, \bbeta ^ *), 1)}. 
\]
The following proposition considers a special case of two groups ($K_0 = 2$) and shows that the RHS of the inequality above is well bounded from below, so that $r(\bX, \bZ, K_0 ,s_0) \lesssim 1$ if $n ^ {-1}\max_{j \in [p]} \|\bX_j\|_2 ^ 2 \lesssim 1$. 
\begin{prop}\label{prop:beta}
Under the RE condition above, we have for any $s_0 \ge 2$ that
\[
    \min_{\substack{{\btheta^* \in \wt\bTheta(2, s_0)}}} \min_{\substack{{\btheta \in \bTheta(|\mathbb{G}(\bbeta ^ *)|, \|\bbeta ^ *\|_0)} \\ {\mathbb{G}(\bbeta) \neq \mathbb{G}(\bbeta ^ *)}}}\frac{ \|\bbeta - \bbeta ^ *\|_2 ^ 2}{\max(d(\bbeta, \bbeta ^ *), 1)} \gtrsim 1.
\]
\end{prop}
We emphasize that the well-separated signals and balanced group sizes in the definition of $\wt\bTheta$ are both essential to guarantee the conclusion of Proposition \ref{prop:beta}. Violating either of the two constraints in the definition of $\wt\bTheta$ can let the double minimum above vanish asymptotically as $s_0 \to \infty$.

\section{Mixed integer optimization formulation}\label{sec:mixed}
Given the strong statistical guarantee established for \lof in the previous section, we now switch our focus to the computational aspect of the problem.
In this section, we leverage mixed integer optimization techniques to solve the combinatorial problem \eqref{eq:prob}. Recently, \cite{bertsimas2016best} proposed a MIO approach to solve the best subset selection problem of a remarkably enhanced scale. This inspires us to formulate \lof as a MIO problem, for which we can resort to modern MIO solvers. In Section \ref{sec:mio_l0}, we introduce the MIO formulation of \lof. Then we present a warm start algorithm in Section \ref{sec:warm} to further accelerate the MIO solver.

% Despite the minimax optimal rate of $0$-$1$ grouping risk, \lof itself is NP-hard due to the cardinality and structural constraints. Similar challenge also arises in the BSS problem \citep{natarajan1995sparse}, which is a special case of \lof as discussed. To tackle this challenge in BSS,  
%We present a brief overview of MIO, including the simply astonishing advances it has enjoyed in the last twenty-five years. We then introduced the proposed MIO formulations for the homogeneity fusion problem, 

\subsection{MIO formulations for homogeneity fusion}\label{sec:mio_l0}

Generally speaking, a MIO problem is formulated as follows:
\begin{eqnarray}
\label{eq:mio}
&& \hspace{-0.3cm}\min_{\balpha \in \RR ^ m} \hspace{0.06cm} \balpha^{\top}\bQ\balpha + \balpha^{\top}\ba \\
&&\text{s.t.} \hspace{0.2cm} \bA\balpha \leq \bb, \nonumber\\
&&\hspace{0.7cm} \alpha_j \in \{0, 1\}, \hspace{0.3cm} j \in \mathcal{I}, \nonumber\\
&&\hspace{0.7cm} \alpha_j \geq 0, \hspace{0.9cm} j \notin \mathcal{I}, \nonumber 
\end{eqnarray} 
where $\ba \in \mbR ^ {m}, \bA \in \mbR^{h \times m}$, $\bb \in \mbR^{h}$, and $\bQ \in \mbR^{m \times m}$ is positive semi-definite.  The symbol ``$\leq$" represents element-wise inequalities. $\mcI$, an index subset of $[m]$, identifies the binary components of $\balpha$. The mixture of discrete and continuous components of $\balpha$ justifies the name of mixed integer programming. For more comprehensive background of MIO, we refer the readers to \cite{bertsimas2005optimization} and \cite{junger2013facets}. Some popular MIO solvers include CPLEX, GLPK, MOSEK and GUROBI. Thanks to the branch-and-bound techniques \citep{cook1995combinatorial}, these solvers can provide both feasible solutions and lower bounds of the optimal objective value, from which we can learn how far a current solution is from the global optimum. 

% Cutting plane theory \cite[]{dantzig1954solution, gouonr1958outline, gomory1960solving}, disjunctive programming for branching rules \cite[]{markowitz1957solution, eastman1958linear, land1960automatic}, improved heuristic methods \cite[]{berthold2006primal}, techniques for preprocessing MIOs \cite[]{savelsbergh1994preprocessing}, using linear optimization methods have all contributed greatly to the speed imptovements in MIO solvers.  Branch-and-cut search is a complete procedure designed to find the optimal solution of a given problem instance or prove infeasibility thereof. See  for a review on the history of branch-and-bound. Preprocessing, or presolving, means to transform a given problem instance into a different but equivalent problem instance that is hopefully easier to solve by the subsequently invoked solution algorithm. In contrast, the goal of primal heuristics is to find good feasible solutions quickly. It is often sufficient in practice to provide a good solution whereas a proof of optimality may not even be computationally tractable. MIO solvers  As the MIO solver progresses toward the optimal solution, the lower bounds improve and provide an increasingly better guarantee of suboptimality, which is especially useful if the MIO solver is stopped before reaching the global optimum. In contrast, heuristic methods do not prove such a certificate of suboptimality. The detailed introduction about the foregoing algorithms can be found in \cite{wolsey2008mixed}.}

Now we  introduce the MIO formulation for problem (\ref{eq:prob}):
%We first present a simple reformulation of problem~\eqref{eq:prob} as a MIO problem: 
\begin{eqnarray}
\label{eq:problem}
&&\hspace{-1.03cm}\min_{\substack{\balpha \in \RR ^ q, \bbeta \in \RR ^ p, \\ \bgamma \in \RR ^ {K}, \bOmega \in \{0, 1\} ^ {p \times (K + 1)}}}\sum_{i=1}^n(y_i - \bx_i^{\top}\bbeta - \bz_i^\top \balpha)^2, \\
&& \text{subject to:} \hspace{0.5cm} \omega_{jk} \in \{0, 1\}, ~~\forall k \in \{0\} \cup[K], j \in [p] \nonumber\\
%\label{eq:cone}
&&\hspace{2.1cm} \omega_{jk}(\beta_j - \gamma_k) = 0, ~~\forall k \in [K], j \in [p] \nonumber\\
&&\hspace{2.1cm} \omega_{j0}\beta_j = 0, ~~\forall j \in [p] \nonumber\\
&&\hspace{2.1cm} \gamma_k < \gamma_{k+1}, ~~\forall k \in [K - 1] \nonumber\\
%\label{eq:order}
%\label{eq:sos1}
&&\hspace{2.1cm} \sum_{k = 0}^{K}\omega_{jk} = 1, ~~\forall j \in [p] \nonumber\\ 
&&\hspace{2.1cm} \sum_{j = 1}^p \omega_{j0} \geq p - s. \nonumber
%\label{eq:sparsi}
%\label{eq:beta1}
\end{eqnarray} 
Here the number of groups $K$ and the sparsity $s$ are prespecified, which will be tuned by, for example, cross-validation. For any $j \in [p]$ and $k \in \{0\} \cup [K]$, we use $\omega_{jk}$ to denote the $(j, k + 1)$ entry of $\bOmega$. For any $j \in [p]$ and $k \in [K]$, $\omega_{jk} = 1$ ($\omega_{jk} = 0$) means that the $j$th covariate is (not) in the $k$-th group. To see why this is true, note that  $\omega_{jk}(\beta_j - \gamma_k) = 0$ enforces $\beta_j = \gamma_k$ when $\omega_{jk} = 1$. Similarly, $\omega_{j0} = 1$ implies that $\beta_j = 0$, given the constraint that  $\omega_{j0}\beta_j = 0$. These types of constraints correspond to Specially Ordered Sets of type 1 (SOS-1) in \cite{beale1970special} and can be replaced by linear constraints \citep{vielma2011modeling,markowitz1957solution,dantzig1960significance}. The constraint $\gamma_k < \gamma_{k+1}$ resolves the identifiability issue so that $\{\gamma_{k}\}_{k \in [K]}$ can be uniquely determined. $\sum_{k=1}^{K}\omega_{jk} = 1$ implies that each covariate belongs to exactly one group. Finally, $\sum_{j=1}^p\omega_{j0} \geq p - s$ ensures the size of the zero-valued group to be bigger than $p - s$, thereby constraining the sparsity of $\bbeta$ below $s$. It is noteworthy that the  solution of problem~\eqref{eq:problem} can have fewer than $K$ groups.  

%    $\gamma_0 = 0$  indicates that all the members in the $0$-th group  have exactly zero effects on the response.   
% Any aforementioned MIO package can solve the problem above. 

Problem \eqref{eq:problem} can be easily extended to accommodate prior knowledge regarding  group structures. For instance, some covariates are known in advance to be in the same group, \textit{say}, $\beta_{j \in \mathcal{J}}$ are equal for a set $\cJ \subset [p]$. Then we can incorporate this information into \eqref{eq:problem} by adding the constraint that $\omega_{j_1k} = \omega_{j_2k}, \forall j_1, j_2 \in \cJ, j_1 \neq j_2, k \in \{0\} \cup [K]$. Another example is that we know no pair of covariates among $\{X_j\}_{j\in \mathcal{J}}$ should belong to the same group. Then we can add the constraint that $\sum_{j \in \mathcal{J}}\omega_{jk} \leq 1, k = \{0\} \cup [K]$.

\subsection{Warm start algorithm}\label{sec:warm}
% Although the constant bound $M$ in the big-$M$ formulation is optional, it can accelerate the convergence of the MIO algorithm. In other words, formulations with tightly specified bounds provide better lower bounds to the global optimization problem in a specified amount of time, when compared to a MIO formulation with loose bound specification. 
This section introduces a discrete first-order algorithm to provide a warm start for the MIO problem \eqref{eq:problem}.  Our algorithm is inspired by \cite{bertsimas2016best}, who proposed a similar algorithm to initialize a MIO solver to solve the BSS problem. Since this algorithm is not limited to the square loss objective in the \lof problem, we extend the original \lof problem to embrace a wider range of objective functions. 

Suppose we are interested in a convex objective function $g(\btheta)$ satisfying that: 
\begin{enumerate}[label=(\roman*)]
    \item $g(\btheta) \geq C_2 > -\infty$ for some universal constant $C_2$;\label{cond:bounded}
    \item $g(\btheta)$ has Lipschitz continuous gradient, i.e., $\|\nabla g(\btheta) - \nabla g(\tilde{\btheta})\|_2 \leq l\|\btheta - \tilde{\btheta}\|_2$ for some positive $l$ and any $\btheta, \tilde\btheta \in \bTheta(K, s)$, which is defined in the beginning of Section \ref{sec:stat_theory}.\label{cond:Lipschitz}
\end{enumerate}
 Consider the following generalized \lof problem: 
\begin{eqnarray}
\label{eq:problemws}
&&\hspace{-1.4cm}\min_{\btheta \in \bTheta(K, s)} \hspace{0.02cm}~ g(\btheta). 
% &&\hspace{-1.2cm}\text{subject to:} \hspace{0.2cm} \beta_j \in \{0, \gamma_1, \dots, \gamma_K \}, \forall j \in [p] \nonumber\\
% &&\hspace{0.55cm}\|\bbeta\|_0 \leq s, \nonumber 
\end{eqnarray} 
% where $\btheta = (\bbeta^{\top}, \balpha^{\top})^{\top}$. 
% Recall that $\bTheta(K, s)  = \{\btheta:  \beta_j \in \{0, \gamma_1, \dots, \gamma_K\}, \forall j \in [p]; \|\bbeta\|_0 \leq s\}.$
% It is obvious that $\bTheta(K, s)$ is a closed set.
We propose an algorithm to attain a feasible point close to the solution of problem \eqref{eq:problemws}, based on ideas from projected gradient descent methods \citep{nesterov2004introductory,nesterov2013gradient}. Note that this point can serve as a starting point for MIO solvers and the objective function value at this point is an upper bound of the global minimum. To do so, we construct a curve $h_L(\btheta,\btheta')$ defined in the following proposition, which lies above $g(\btheta)$ and is tangent to $g(\btheta)$ at $\btheta'$: 
\begin{prop}[\cite{nesterov2004introductory,nesterov2013gradient}]
\label{prop:GCurve}
For a convex function $g(\btheta)$ satisfying \ref{cond:Lipschitz}, and for any $L\geq l$, we have:
\begin{equation}
    g(\btheta)\leq h_L(\btheta,\btheta'):=g(\btheta')+(\btheta-\btheta')^\top
    \nabla g(\btheta')+\frac{L}{2}\norm{\btheta-\btheta'}_2^2\label{def:GCurve}
\end{equation}
for all $\btheta,\btheta'$ with equality holding at $\btheta=\btheta'$.
\end{prop}

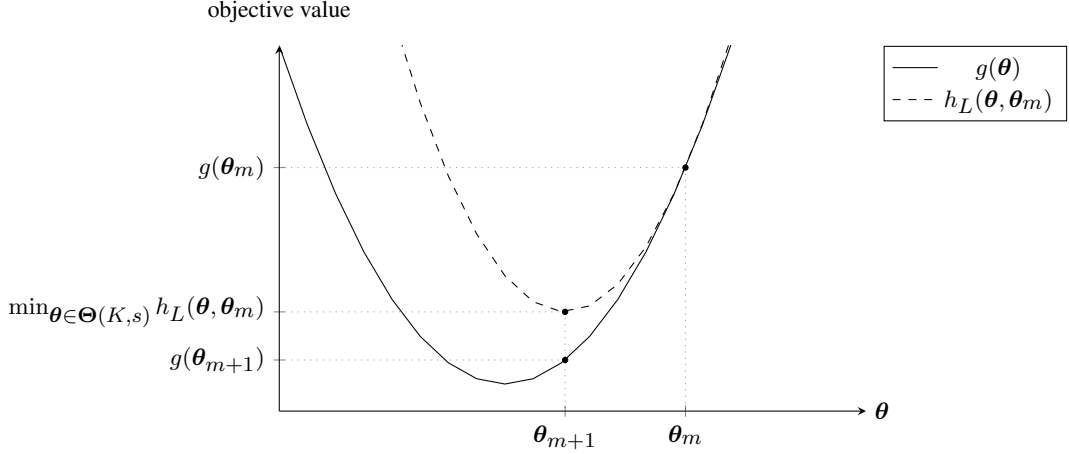
\begin{figure}[h]
\begin{tikzpicture}
\begin{axis}[
    x=0.6cm,
    y=0.18cm,
    % width and height if axis, adjust to your liking
    xtick={4/3,4}, % remove all ticks from x-axis
    xticklabels={$\btheta_{m+1}$,$\btheta_m$},
    ytick={106/9,46/3,26}, % ditto for y-axis
    yticklabels={$g(\btheta_{m+1})$,$\min_{\btheta\in\bTheta(K,s)}h_L\paren{\btheta,\btheta_m}$,$g(\btheta_m)$},
    xlabel=$\btheta$, 
    ylabel={objective value},
    xmin=-5,
    xmax=8,
    ymin=8,
    ymax=35,
    axis lines=left,
    every axis x label/.style={at={(current axis.right of origin)},anchor=west},
every axis y label/.style={at={(current axis.north west)},above=2mm},
    legend pos=outer north east
    ]
  \addplot [solid][domain=-5:10] {x^2+10};
  \addlegendentry{$g(\btheta)$}
  \addplot [dashed][domain=-5:10] {1.5*x^2-4*x+18};
  \addlegendentry{$h_L(\btheta,\btheta_m)$}
  
  \addplot [only marks,black,mark size=1pt] coordinates {(4,26)};
  \addplot [only marks,black,mark size=1pt] coordinates {(4/3,46/3)};
  \addplot [only marks,black,mark size=1pt] coordinates {(4/3,106/9)};
  \draw [dotted, opacity=0.4] (axis cs:{4,26}) -- (axis cs:{4,0});
  \draw [dotted, opacity=0.4] (axis cs:{4/3,46/3}) -- (axis cs:{4/3,0});
  \draw [dotted, opacity=0.4] (axis cs:{4,26}) -- (axis cs:{-5,26});
  \draw [dotted, opacity=0.4] (axis cs:{4/3,46/3}) -- (axis cs:{-5,46/3});
  \draw [dotted, opacity=0.4] (axis cs:{4/3,106/9}) -- (axis cs:{-5,106/9});
\end{axis}
\end{tikzpicture}

    \caption{An illustration of $g(\btheta)$ and $h_L(\btheta,\btheta_m)$ in Proposition~\ref{prop:GCurve}. The solid curve is $g(\btheta)$ and the dashed curve is $h_L(\btheta,\btheta_m)$.}
    \label{fig:GCurve}
\end{figure}

As illustrated in Figure~\ref{fig:GCurve}, given a point $\btheta_m$, we can always improve the current objective value $g(\btheta_m)$ through the descending route: 
\begin{equation}
g(\btheta_{m+1})\leq h_L(\btheta_{m+1},\btheta_m)\leq h_L(\btheta_{m},\btheta_{m})=g(\btheta_m), \label{eq:decreasingSequence}
\end{equation}
where 
\[
\btheta_{m+1}\in\argmin_{\btheta\in\bTheta(K,s)}h_L(\btheta,\btheta_m)=\argmin_{\btheta\in\bTheta(K,s)}\norm[\Big]{\btheta-\paren[\big]{\btheta_m-\frac{1}{L}\nabla g(\btheta_m)}}_2^2.
\] For convenience, for any constant vector $\bc = (c_1, \ldots, c_{p + q})^\top$, define
\[
    \cH_{K, s}(\bc) := \argmin_{\btheta \in \bTheta(K, s)} \|\btheta - \bc\|_2 ^ 2. 
\]
Then $\btheta_{m+1}\in\cH_{K, s}(\btheta_m-\frac{1}{L}\nabla g(\btheta_m))$.
By doing this improvement iteratively, we implement Algorithm \ref{alg:warm_start} below that supplies our warm starts to solve problem \eqref{eq:problemws}.

\begin{center}
    \begin{algorithm}[h]
        \caption{Warm Start}
        \label{alg:warm_start}
        \KwIn{Loss function $g(\btheta)$, number of groups $K$, sparsity constraint $s$, step size parameter $L$ and convergence tolerance $\varepsilon$.}
        \begin{algorithmic}[1]
            \State Initialize with $\btheta_1 \in \mathbb{R}^{p+q}$.
            \State For $m \geq 1$, $\btheta_{m+1} \in \cH_{K, s}(\btheta_m - \frac{1}{L}\nabla g(\btheta_m))$.
            \State Repeat Step~2 until $g(\btheta_m) - g(\btheta_{m+1}) \leq \varepsilon$.
            % \State Return $\btheta_{m+1}$.
        \end{algorithmic}
        \KwOut{$\btheta_{m+1}$.}
        % \KwOut{A feasible point $\btheta^*$ such that $g(\btheta^*) = g(\btheta)$ where $\btheta$ is a first-order stationary point.}
    \end{algorithm}
\end{center}

Algorithm \ref{alg:warm_start} is essentially a projected gradient descent algorithm: In each iteration, we perform a gradient descent step followed by projection onto $\cH_{K,s}$. 
To obtain an element in $\cH_{K, s}(\bc)$ for any $\bc\in\RR^{p+q}$, we can exploit the subroutine Algorithm~\ref{alg:0} in Appendix~\ref{subsection:appendixAlg}, which is a generalization of the segment neighbourhood method \citep{auger1989algorithms} with sparsity constraint. To investigate the algorithmic convergence of Algorithm \ref{alg:warm_start}, we first define the first-order stationary points of problem \eqref{eq:problemws} as follows.
% let $\cH_{K, s}(\bc)$ denote the set of optimal solutions to problem~\eqref{eq:problemws} with $g(\btheta) = \|\btheta - \bc\|^2_2$. For any $(\hat{\btheta}^{\top}, \hat{\bgamma}^{\top})^{\top} \in \cH_{K, s}(\bc)$, we have $\hat{\bgamma} = \{\text{all different values in} \hspace{0.1cm} \hat{\bbeta}\}$. That is, $\hat{\bgamma}$ could be uniquely determined by the values of $\hat{\bbeta}$. { \color{red} For the ease of presentation, we generally omit parameters $\hat{\gamma}$, and say $\hat\btheta \in \cH_{K, s}(\bc)$ if $\hat\btheta$ is the solution minimizing the problem~\eqref{eq:problemws}, except where necessary.  ?}
\begin{defn}[First-order stationary point]
	\label{def:1}
	We say a vector $\btheta \in \bTheta(K, s)$ is a first-order stationary point for problem \eqref{eq:problemws} if $\btheta \in \cH_{K, s}(\btheta - \frac{1}{L}\nabla g(\btheta))$ for some positive constant $L \geq l$. 
\end{defn}
%\begin{defn}
%	\label{def:2}
%	($\varepsilon$-approximate first-order stationary point). Given problem~\eqref{eq:problemws}, two positive constants $\varepsilon$ and $L \geq l$, we say that $\btheta \in \bTheta(K, s)$ is an $\varepsilon$-approximate first-order stationary point if for some $\btheta' \in H_{K, s}(\btheta - \frac{1}{L}\nabla g(\btheta))$ we have $\|\btheta - \btheta'\|_2 \leq \varepsilon$. {\color{red} Why we need this?}
%\end{defn}
% The following proposition shows that the optimal solution set has only one first-order stationary point. 
The following proposition establishes two important properties of the first-order stationary points that underpin the effectiveness and stability of our warm start Algorithm~\ref{alg:warm_start}. 
\begin{prop}
	\label{prop:1}
	Suppose a positive constant $L > l$. 
	\begin{itemize}
		\item[1.] If $\btheta$ is a solution to problem~\eqref{eq:problemws}, then it is a first-order stationary point. 
		\item[2.] If $\btheta$ is a first-order stationary point, then the set $\cH_{K, s}(\btheta - \frac{1}{L}\nabla g(\btheta))$ has exactly one element $\btheta$.
	\end{itemize}
\end{prop}
\begin{proof}
    \begin{itemize}
        \item[1.] From \eqref{eq:decreasingSequence} and that $\btheta$ is a solution to problem~\eqref{eq:problemws}, we know $\min_{\btheta'\in\bTheta(K,s)}\allowbreak h_L(\btheta',\btheta)=h_L(\btheta,\btheta)$. Then $\btheta\in\cH_{K,s}(\btheta-\frac{1}{L}\nabla g(\btheta))$.
        \item[2.] Assume $\tilde\btheta\in\cH_{K,s}(\btheta-\frac{1}{L}\nabla g(\btheta))$ and $\tilde\btheta\neq\btheta$. Then $(\tilde\btheta-\btheta)^\top\nabla g(\btheta)=\frac{L}{2}\norm{\tilde\btheta-\btheta}_2^2>0$. Since $g$ is convex, we have $g(\tilde\btheta)>g(\btheta)$. This contradicts with \eqref{eq:decreasingSequence}.
    \end{itemize}
\end{proof}
% We proceed to show that Algorithm \ref{alg:warm_start} is capable of finding a feasible point whose objective function value is the same as that of a first-order stationary point of problem~\eqref{eq:problemws}.
%\noindent \textbf{Algorithm 1.} \\
%{\em Input:} $g(\btheta)$, number of groups: K, sparsity constraint: s, parameter: L and convergence tolerance: $\varepsilon$. \\
%{\em Output:} A feasible point $\btheta^*$ such that $g(\btheta^*) = g(\btheta)$ where $\btheta$ is a first-order stationary point.
%\begin{itemize}
%	\item[1.] Initialize with $\btheta_1 \in \mathbb{R}^{p+q}$.
%	\item[2.] For $m \geq 1$, $\btheta_{m+1} \in H_{K, s}(\btheta_m - \frac{1}{L}\nabla g(\btheta_m))$.
%	\item[3.] Repeat Step~2, until $g(\btheta_m) - g(\btheta_{m+1}) \leq \varepsilon$.
%	\item[4.] Return $\btheta_{m+1}$.
%\end{itemize}

Now we present the convergence property and convergence rate of Algorithm \ref{alg:warm_start} through Proposition~\ref{prop:3} and Theorem~\ref{the:conver}, respectively.
\begin{prop}
	\label{prop:3}
	For problem~\eqref{eq:problemws} and some positive constant $L>l$, let $\btheta_m, m \geq 1$ be the sequence generated by Algorithm \ref{alg:warm_start}. We have 
	\begin{enumerate}[label={\arabic*}.,ref=\arabic*]
		\item $g(\btheta_m) - g(\btheta_{m+1}) \geq \frac{L - l}{2}\|\btheta_m - \btheta_{m+1}\|^2_2$; \label{statement:gDifference} 
		\item  $\|\btheta_{m+1} - \btheta_{m}\|_2 \to 0$ as $m \to \infty$. \label{statement:shrink}
	\end{enumerate}
\end{prop}
\begin{proof}
The first statement holds because
            \begin{align*}
                &g(\btheta_m)=h_L(\btheta_m,\btheta_m)\geq h_L(\btheta_{m+1},\btheta_m)\\
               =&g(\btheta_m)+(\btheta_m-\btheta_{m+1})^\top\nabla g(\btheta_m)+\frac{L}{2}\norm{\btheta_m-\btheta_{m+1}}_2^2~~~\text{(From \eqref{def:GCurve})}\\
               =&g(\btheta_m)+(\btheta_m-\btheta_{m+1})^\top\nabla g(\btheta_m)+\frac{l}{2}\norm{\btheta_m-\btheta_{m+1}}_2^2+\frac{L-l}{2}\norm{\btheta_m-\btheta_{m+1}}_2^2\\
               \geq& g(\btheta_{m+1})+\frac{L-l}{2}\norm{\btheta_m-\btheta_{m+1}}_2^2. ~~~ \text{(From Proposition~\ref{prop:GCurve})}
            \end{align*}
To prove the second statement, we note that from \eqref{eq:decreasingSequence} and Condition \ref{cond:bounded}, $\brac{g(\btheta_m)}_{m=1}^\infty$ is decreasing and bounded from below, so it is convergent. Then $\lim_{m\to\infty}\norm{g(\btheta_{m+1})-g(\btheta_m)}_2^2=0$. From Proposition~\ref{prop:3} Statement~\ref{statement:gDifference}, we have $\lim_{m\to\infty}\norm{\btheta_{m+1}-\btheta_m}_2^2=0$.

\end{proof}
\begin{thm}
	\label{the:conver}
	For the sequence $\{\btheta_m\}_{m=1}^\infty$ generated by Algorithm \ref{alg:warm_start}, if $L>l$, then there exists $c\in\mathbb{R}$, such that for any $M\in\mathbb{Z}^+$ we have
	$$
	    \min_{m = 1, \dots, M}\|\btheta_{m+1} - \btheta_{m}\|^2_2 \leq \frac{2(g(\btheta_1) - c)}{M(L - l)},
	$$
	where $g(\btheta_m)\downarrow c $ as $m \to \infty$. 
\end{thm}
\begin{proof}
    From \eqref{eq:decreasingSequence} and Condition \ref{cond:bounded}, we have the fact that $\brac{g(\btheta_m)}_{m=1}^\infty$ is decreasing and bounded from below, so it is convergent to some $c\in\mathbb{R}$. For this $c$, the conclusion follows directly from Proposition~\ref{prop:3} Statement~\ref{statement:gDifference}.   
\end{proof}
Finally, we show that Algorithm \ref{alg:warm_start} gives a feasible solution whose objective value is the same as some first-order stationary point under mild conditions:
\begin{prop}
 	\label{prop:4}
 	Consider problem~\eqref{eq:problemws} and some constant $L>l$, let $\btheta_m, m \geq 1$ be the sequence generated by Algorithm \ref{alg:warm_start}. Suppose $g$ satisfies the following conditions:  
	\begin{enumerate}
		\item
			$g$ has second-order derivative; 
		\item
			there exists $l'>0$ such that $l'\big\Vert \btheta-\tilde\btheta \big\Vert_2\leq \big\Vert \nabla g(\btheta)-\nabla g(\tilde\btheta) \big\Vert_2$ for any $\btheta,\tilde\btheta\in\Theta(K,s)$ satisfying $\mathbb{G}(\bbeta)=\mathbb{G}(\tilde\bbeta)$; 
		\item
			$\{\btheta \in \bTheta(K, s)\,|\,g(\btheta) \leq C\}$ is bounded for any $C \in \mathbb{R}$. 
	\end{enumerate}
	Then $g(\btheta_m)$ converges to $g(\btheta)$ where $\btheta$ is a first-order stationary point.
\end{prop}
The detailed proof of Proposition~\ref{prop:4} is given in Proposition~\ref{prop:2} and Remark~\ref{remark:1} in Appendix~\ref{app:1}.

% Here we do not need data splitting as usually expected. In Thoerem \ref{thm:suff}, we show that under t'he sufficient condition, the solution $\wh{\btheta}$ to problem \eqref{eq:problem} exactly recovers the oracle estimator $\wh{\btheta}^{\mathrm{ol}}$. Conditioned on $\cS_0\subseteq \wt{\cS}$, the oracle estimator to the reduced MIO formulation remains the same as the original MIO formulation. Thus we do not suffer from the dependency between the screening stage and the grouping stage. 

\section{Numerical studies}\label{sec:num}
We conduct a variety of numerical experiments to assess the performance of \lof\hspace{-.15cm}. 
We use the \textit{normalized mutual information} (NMI, \cite{ana2003robust}) to evaluate grouping accuracy. Specifically, given $\mathbb{G}_1 = \{G_1^{(1)},G_1^{(2)},... \}$ and $\mathbb{G}_2 = \{G_2^{(1)}, G_2^{(2)}, ...\}$ as two sets of disjoint clusters of $[p]$, define the mutual information $I(\mathbb{G}_1;\mathbb{G}_2)$ between $\mathbb{G}_1$ and $\mathbb{G}_2$ as
\[
I(\mathbb{G}_1;\mathbb{G}_2) := \sum_{i \in [|\mathbb{G}_1|], j \in [|\mathbb{G}_2|]}\frac{\big|G_1^{(i)}\cap G_2^{(j)}\big|}{p}\log\bigg(\frac{p\big|G_1^{(i)}\cap G_2^{(j)}\big|}{\big|G_1^{(i)}\big|\big|G_2^{(j)}\big|}\bigg), 
\]
and define the entropy of $\mathbb{G}_1$ as
\[
    H(\mathbb{G}_1) := I(\mathbb{G}_1;\mathbb{G}_1) = - \sum_{i \in [|\mathbb{G}_1|]}\frac{|G_1^{(i)}|}{p}\log \bigg(\frac{|G_1^{(i)}|}{p}\bigg). 
\]
Now we are ready to define the NMI between $\mathbb{G}_1$ and $\mathbb{G}_2$ as
\[
    \text{NMI}(\mathbb{G}_1,\mathbb{G}_2) := \frac{I(\mathbb{G}_1; \mathbb{G}_2)}{\{H(\mathbb{G}_1) + H(\mathbb{G}_2)\}/2}. 
\]
Note that if $\mathbb{G}_1$ and $\mathbb{G}_2$ share the same group structure, we have $\text{NMI}(\mathbb{G}_1,\mathbb{G}_2) =1$.

The rest of the section is organized as follows. Section \ref{sec:nmi} compares \lof with its competitors in terms of grouping accuracy and investigates the effectiveness of the warm start Algorithm~\ref{alg:warm_start} under low-dimensional regimes. Section \ref{sec:ultra} implements the ``\textit{screening then grouping}'' strategy discussed in Section \ref{sec:ultra_met} to perform homogeneity fusion under ultrahigh-dimensional sparse setups. Finally, Section \ref{sec:real} applies \lof to group lipids in a study of metabolomic effects on body mass index (BMI).

\subsection{Low-dimensional regime}\label{sec:nmi}
We consider a collection of low-dimensional setups where the design vectors $\{\bx_{i}\}_{i=1}^{n}$ are independent realizations from a $p$-dimensional multivariate normal distribution   $\cN(\mathbf{0}, \bSigma)$ with mean zero and covariance matrix $\bSigma := (\Sigma_{ij})$. 
%we set $n = 120, p = 80, q=0, s = 80$ and $K=4$. Rows of 
% design matrix $\bX$ are i.i.d. from $\cN(\textbf{0},\bSigma)$, where 
We adopt the autoregressive design in the sense that $\Sigma_{ij} = \rho^{|i-j|}$
with $\rho\in\{0,0.5\}$. In particular, $\rho=0$ gives the independent design.
For each fixed $\bX$, 
we generate the responses $\by = \bX\bbeta^{*} + \bepsilon$ with $\bepsilon\sim\cN(0,\bI)$. Throughout this section, we always set the group number $K_0 = 4$, while $n, p$ and $\bbeta^{*}$ are specified in the following subsections.

In Sections \ref{sec:eq_group} and \ref{sec:ueq_group}, we compare six methods when group sizes are equal and unequal respectively: \lof\hspace{-.1cm}, ordinary least squares (OLS), fused LASSO (fLASSO, \cite{tibshirani2005sparsity}),  pairwise fusion (pairReg, \cite{ma2017concave}), feature grouping and selection over an undirected graph (FGSG, \cite{zhu2013simultaneous}) and clustering algorithm in regression via data-driven segmentation (CARDS, \cite{ke2016structure}). For \lof\hspace{-.12cm}, CARDS and fLASSO, the tuning parameters are chosen via Bayesian Information Criterion (BIC). For OLS, FGSG and pairReg, we first tune the parameters (if any) in these methods via 10-fold cross-validation in terms of mean squared error (MSE). Note that these methods encourage coefficients within the same group to be close but not exactly the same. To derive grouping structures and gauge their accuracy,  we perform \textit{k-means} clustering on the solutions of OLS, FGSG, pairReg with oracle cluster number $K=4$. We use OLS+, FGSG+, pairReg+ to represent the corresponding post-clustering results. In Section \ref{sec:warm_exp}, we present the NMI of the warm-start solution in Section \ref{sec:warm} with varying $n$ and $p$, and illustrate how warm starts help the convergence of \lof, especially in the early stage. All the results are based on $200$ independent Monte Carlo experiments.

\subsubsection{Equal group sizes}\label{sec:eq_group}

We start with the case where all the coefficient groups have equal sizes. Specifically, we let $p=80$ and have $4$ coefficient groups of size $20$, which take values $-2r, -r,r,2r$ respectively with $r \in\{ 0.5, 0.8\}$.  Figure \ref{fig:equ_nmi} displays the boxplots of NMI for correlation coefficient $\rho\in\{0,0.5\}$ and signal strength $r\in\{0.5,0.8\}$. The results are based on $n=120$ observations. We have the following observations: 
\begin{itemize}
    \item[(i)] \lof exhibits significantly higher NMI than the other methods under all the cases, even though OLS+, pairReg+ and FGSG+ have oracle knowledge of the true number of groups.
    \item[(ii)] Nearly all the methods yield higher grouping accuracy when $r$ is larger (compare left red and right blue boxplots) or $\rho$ is smaller (compare panels (a) and (b)). 
\end{itemize}
%  Moreover, \lof takes advantages of increases in signal strength and achieves perfect group structure recovery when $r=0.8$ under both independent and correlated designs. It appears that with the oracle knowledge of group number $K=4$, OLS+, pairReg+ and FGSG+ also perform well. In contrast, fLASSO tends to improve the model complexity, i.e., select the model with more groups, hence yielding just passable grouping accuracy. CARDS, on the other hand, benefits from the pursuit of parsimony and outperforms other penalized methods.

\begin{figure}
      \centering
      \setlength\tabcolsep{0pt}
      \renewcommand{\arraystretch}{-5}   
      \begin{tabular}{c}
      \includegraphics[width = 11cm]{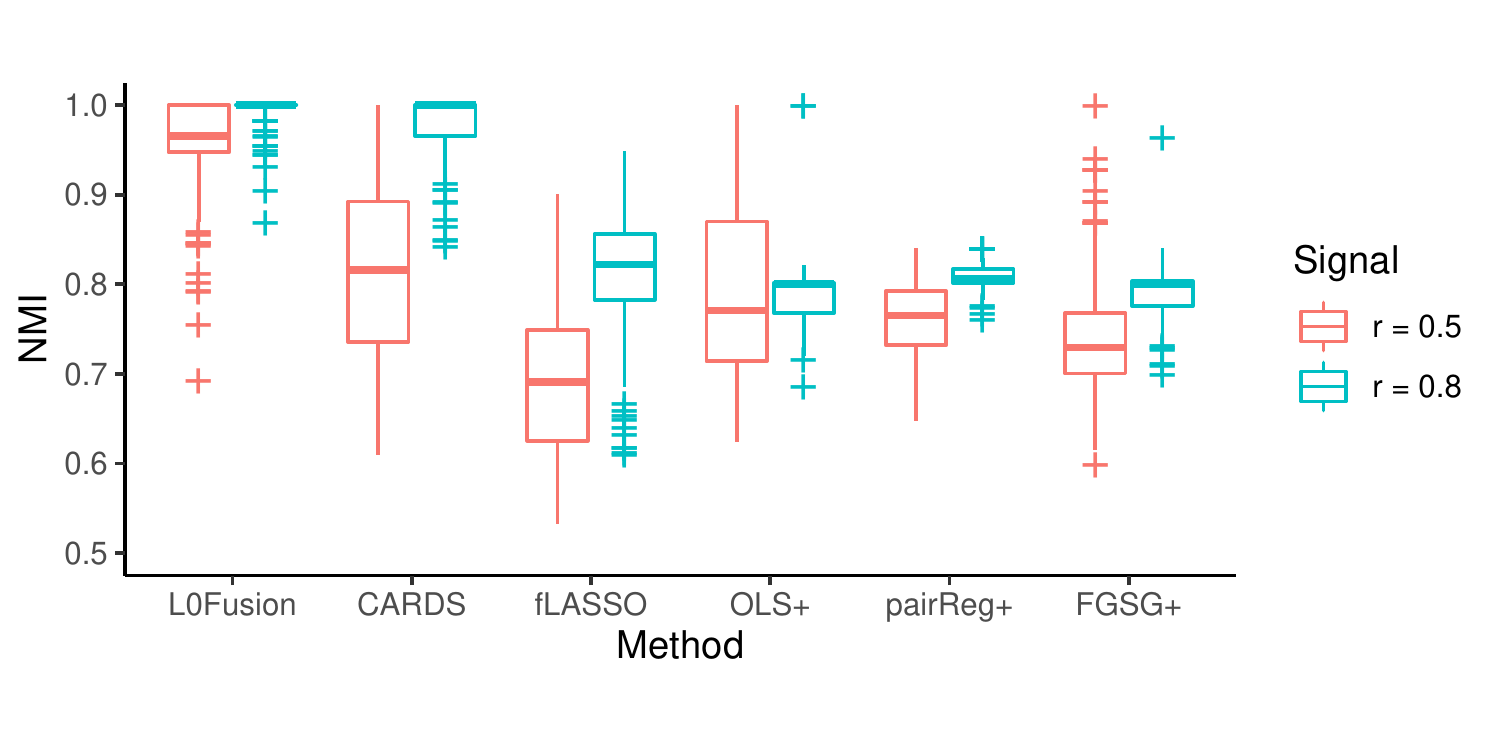} \\
      (a) Independent design \\
      \includegraphics[width = 11cm]{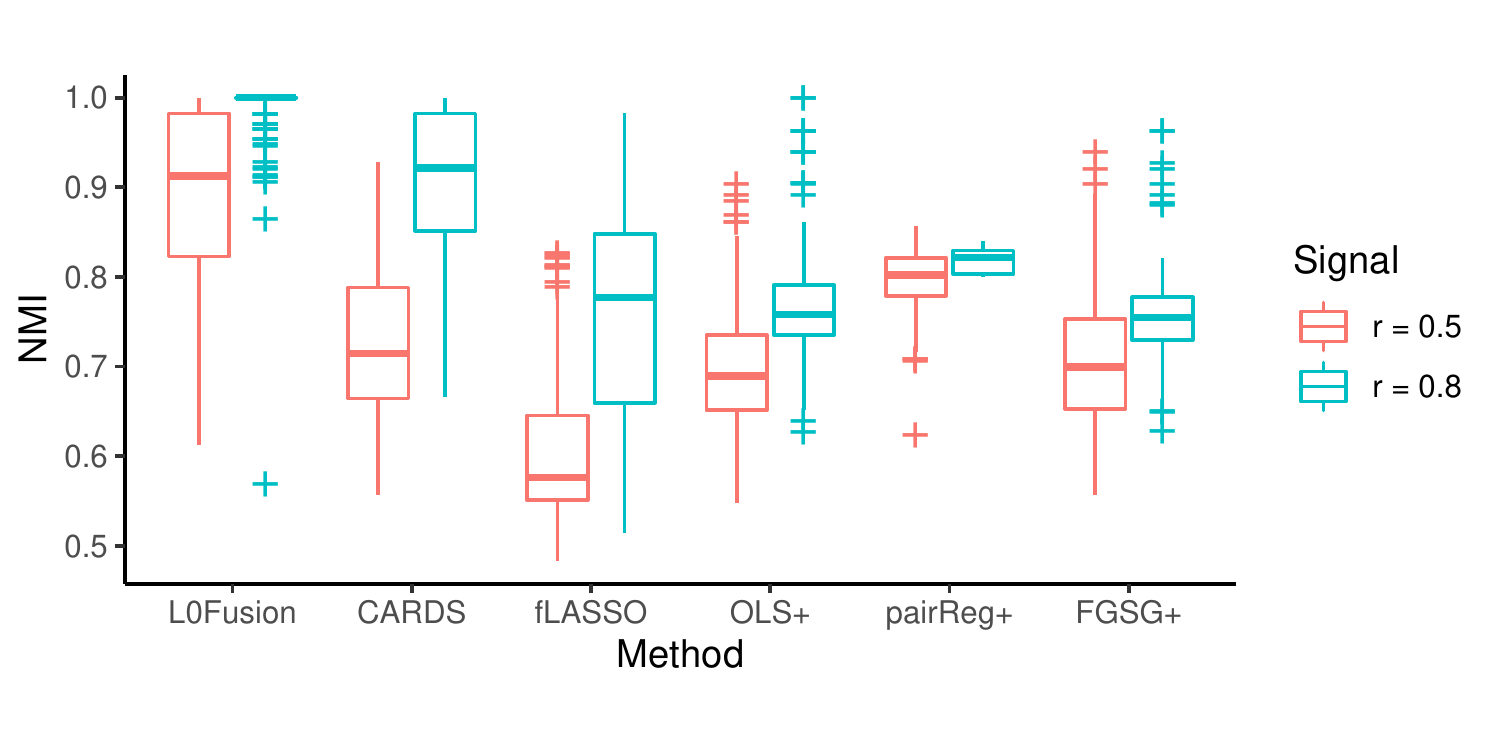}  \\
      (b) Autoregressive design with $\rho = 0.5$
       \end{tabular}
      \caption{Grouping accuracy with equal group sizes under different covariance designs and signal strengths. We set the correlation coefficient $\rho \in\{0,0.5\}$ and signal strength $r\in\{0.5,0.8\}$.  }
      \label{fig:equ_nmi}
\end{figure}

\subsubsection{Unequal group sizes} \label{sec:ueq_group}
 Now we consider groups of different sizes. Particular challenges can arise from identifying small groups whose collective explanation power is typically weak. To assess the capability of detecting small groups, we let $p=80$ and divide the true predictors into $4$ groups of sizes $1,20,20,39$, whose coefficient values are $-4r, -r,r,2r$ respectively.  Figure \ref{fig:unequ_nmi} shows the boxplots of NMI of all the aforementioned approaches for $\rho\in\{0,0.5\}$ and $r\in\{0.5,0.8\}$. The results are based on $n=120$ observations. We have the following observations: 
 \begin{itemize}
    \item[(i)] Similarly to Section \ref{sec:eq_group}, \lof outperforms the competing methods in terms of NMI uniformly under all the cases.
    \item[(ii)] Similarly to Section \ref{sec:eq_group}, all the methods yield higher grouping accuracy when $r$ is larger (compare red and blue boxplots) or $\rho$ is smaller (compare panels (a) and (b)). 
    \item[(iii)] The performance gap between \lof and fLASSO is further enlarged here compared with the case of equal group sizes, which suggests the robustness of \lof with respect to group size heterogeneity. 
\end{itemize}
%   Different from the equal group size cases, pairReg+ and FGSG+ significantly outperform OLS+ relying on their ability to distinguish strong signals in small groups.

\begin{figure}
      \centering
      \begin{tabular}{c}
      \includegraphics[width = 11cm]{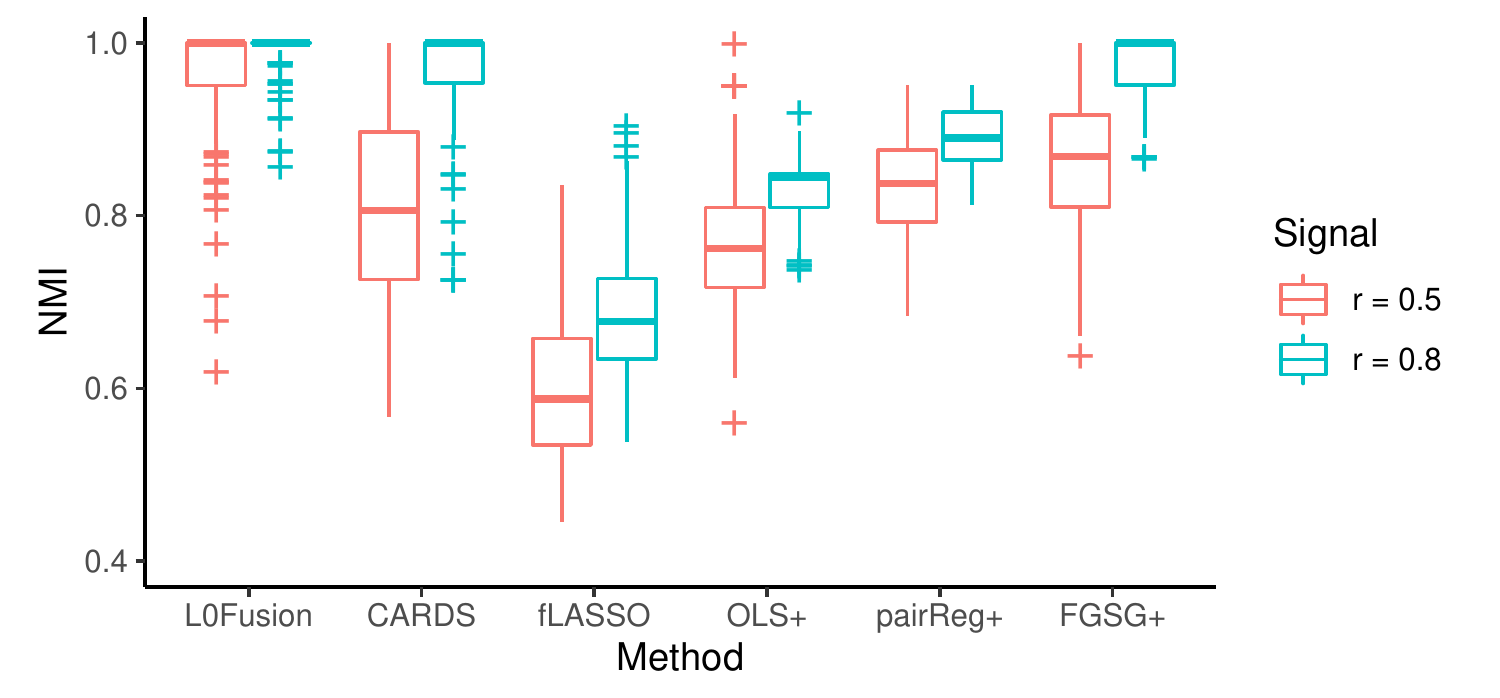}  \\  
      (a) Independent design \\
      \includegraphics[width = 11cm]{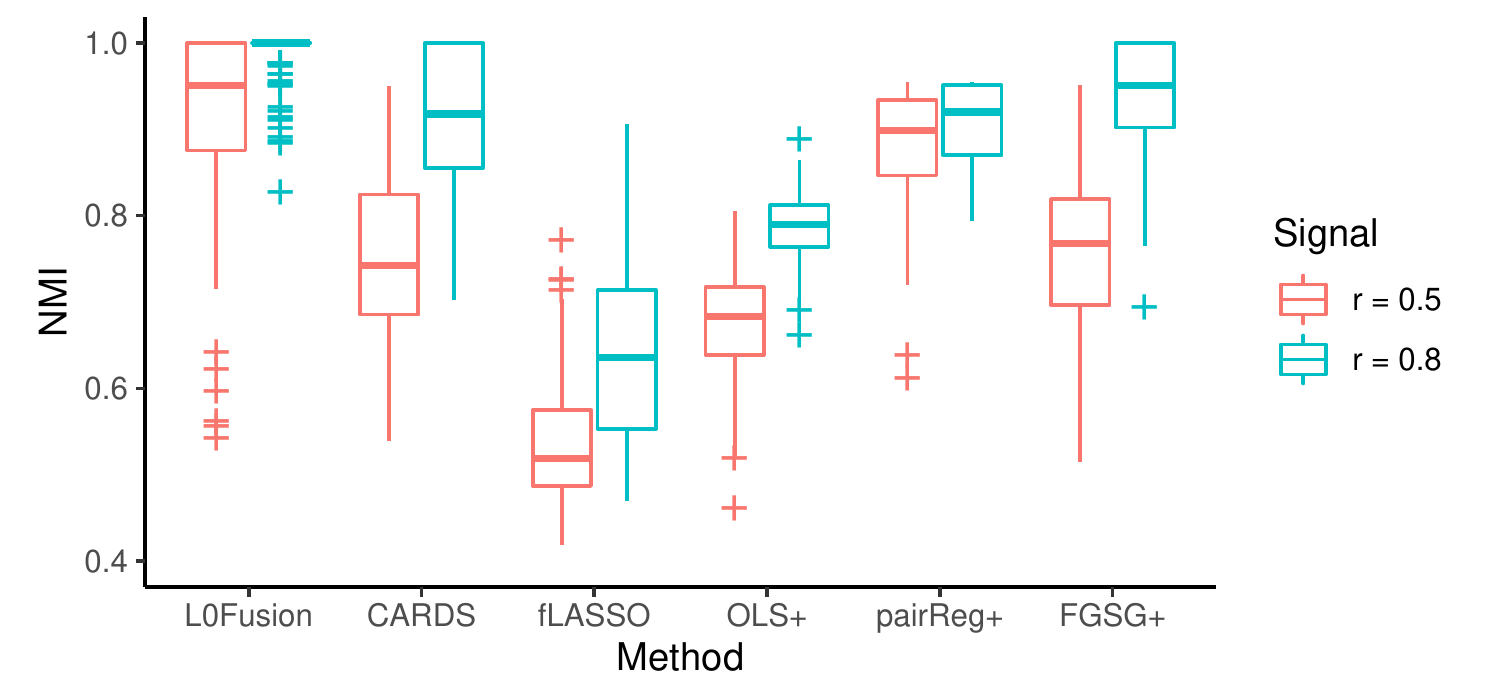}  \\
      
      (b) Autoregressive design with $\rho = 0.5$
       \end{tabular}
      \caption{Grouping accuracies with unequal group sizes under different covariance designs and signal strengths. We set the correlation coefficient $\rho \in\{0,0.5\}$ and signal strength $r\in\{0.5,0.8\}$.}
      \label{fig:unequ_nmi}
\end{figure}

\subsubsection{Warm-start algorithm}\label{sec:warm_exp}
We first assess the grouping accuracy of the solution of the discrete first-order algorithm  introduced in Section \ref{sec:warm} with initial value $\btheta_0=\mathbf{0}_{p+q}$ in simulation. 
% The discrete first-order algorithm can serve for the homogeneity fusion task and can also be implemented as a warm start of \lof.
% We first investigate the sensitivity of the discrete first-order algorithm with respect to sample size ($n$) and dimensionality ($p$). 
Figure \ref{fig:warm_sample} presents the NMI of this algorithm with oracle $K$ as $n$ and $p$ vary. The plot shows its deteriorating performance as $p$ grows or $n$ decreases. However, it is clear that this warm-start algorithm is capable of recovering the group structure with a sufficiently large sample.

\begin{figure}
      \centering
      \setlength\tabcolsep{0pt}
      \includegraphics[width=10.5cm]{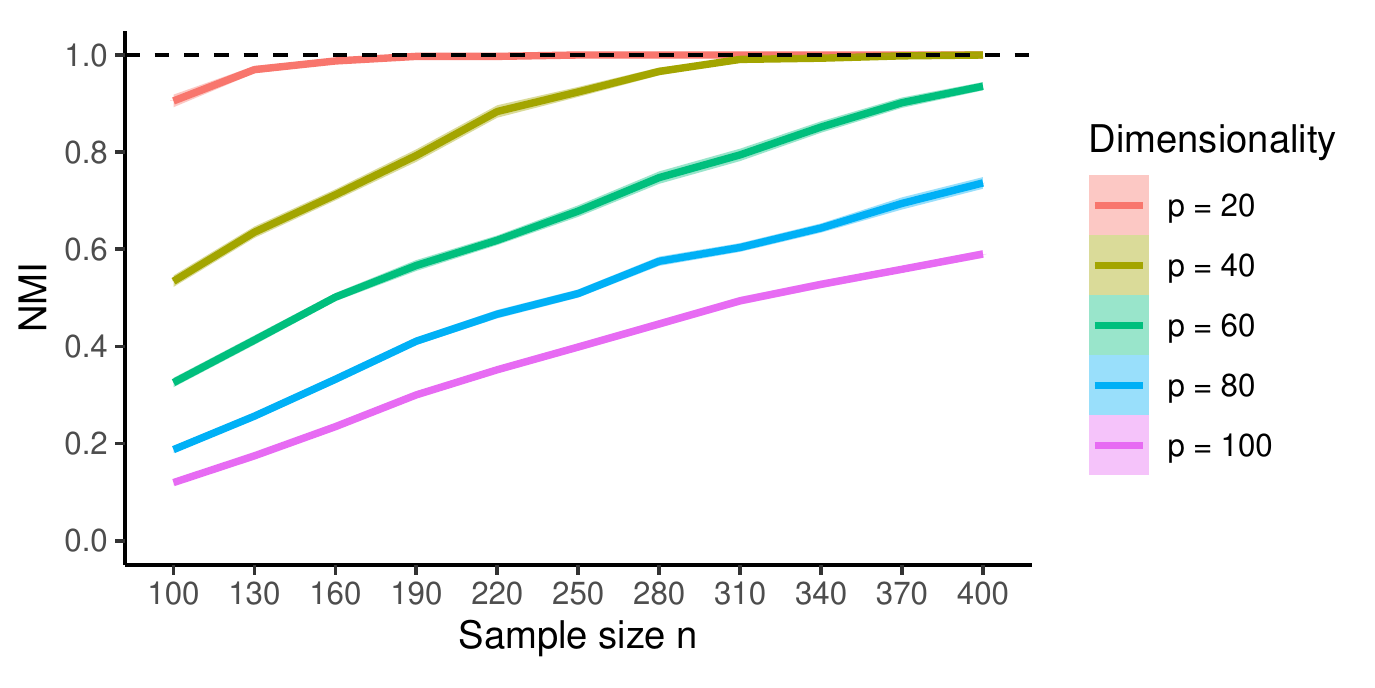} 
      \caption{Error bands for NMIs of the discrete first order algorithm with respect to different sample size and dimensionality. Here $s_0 = p$, $q=0$ and $K_0=4$. Rows of 
 design matrix $\bX$ are i.i.d. from $\cN(\textbf{0},\bI_n)$.
 The true predictors are divided into $4$ groups of size $(p/4,p/4,p/4,p/4)$. True coefficients within each group are $(-1, -0.5,0.5,1 )$ respectively. }
      \label{fig:warm_sample}
\end{figure}

Next, we exploit the discrete first-order algorithm to provide a warm start for $L_0$-Fusion. The MIO solver in \textit{Gurobi} \citep{gurobi} terminates when the gap between the lower and upper objective bounds is less than the Mixed-Integer Programming (MIP) Gap (a user-determined parameter between $0$ and $1$) times the absolute value of the incumbent objective value. More precisely, let $z_P$ be the incumbent primal objective value, which is an upper bound for the global minimum, and $z_D$ be the dual objective value, which is a lower bound for the global minimum. Then the MIP Gap is defined as
$\vert z_P - z_D\vert / \vert z_P\vert$. Figure \ref{fig:warm250} tracks the MIP Gap and the NMI of \lof against its running time on the University of Michigan High Performance Linux cluster. Each job uses $4$ CPUs and $16$ GB memory, which can be satisfied on most personal computers.

\begin{figure}
      \centering
      \setlength\tabcolsep{0pt}
	\renewcommand{\arraystretch}{-5}      
      \begin{tabular}{cc}
      \includegraphics[width=12cm]{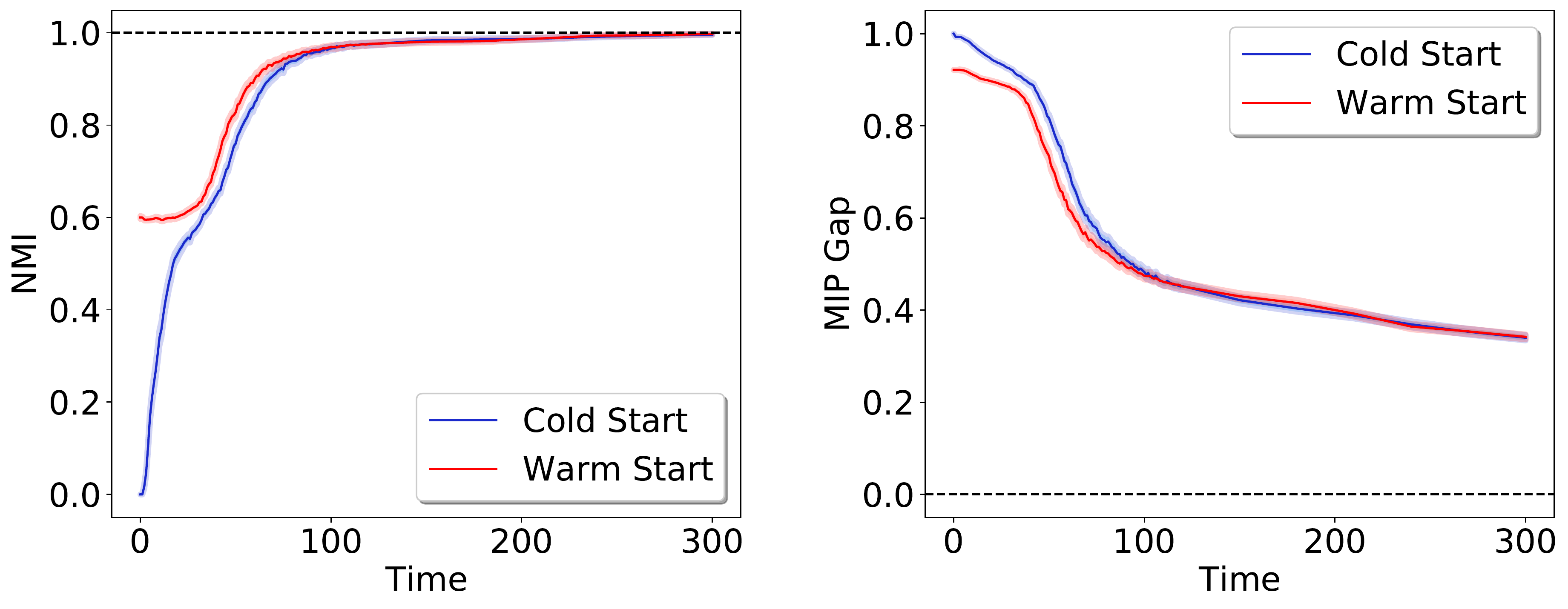} 
      \end{tabular}
      \caption{Error bands for comparing warm start and cold start after $200$ Monte Carlo repetitions. Here we set $n=250$, $p=120$, $s_0 = 120$ and $K_0=4$. Rows of 
 design matrix $\bX$ are i.i.d. from $\cN(\textbf{0},\bI_n)$.
 The true predictors are divided into $4$ groups of size $(30,30,30,30)$. True coefficients within each group are $(-1, -0.5,0.5,1 )$. }
      \label{fig:warm250}
\end{figure}

We have the following three observations from the plots above: 
\begin{itemize}
    \item[(i)] The warm-start solution yields $\mathrm{NMI} = 0.6$ , which is plausible but far from optimal. 
    \item[(ii)] \lof with a warm start yields significantly higher NMI than that with a cold start within the first 50 seconds. 
    \item[(iii)] Even when the MIP Gap is not exactly $0$, \lof can achieve near perfect group recovery. Therefore, one can still expect decent grouping results even if the algorithm has to halt before the MIP Gap vanishes.
\end{itemize}
% First, Second, in such cases, Gurobi itself is capable of searching the solution efficiently and the use of warm start does not distinctly improve the convergence. We think the improvement could be more apparent when $s$ is larger. Third,  Based on this, 

\subsection{Ultra-high dimensional regime}\label{sec:ultra}

For ultra-high dimensional cases, we let $p = 20,000$, $s_0 = 60$, $n=\lfloor 2s\log p\rfloor$ and $K_0=4$.  All the entries of the design matrix $\bX\in\RR^{n\times p}$ are independent standard Gaussian random variables. The true predictors are set as the first $60$ predictors and divided into $4$ groups of size $15$ with coefficient values $-2r, -r,r,2r$ respectively, where $r\in\{0.15, 0.2,0.25, 0.3\}$. All the results in this section are based on $100$ independent Monte Carlo repetitions.

In the screening step, we first estimate the true sparsity $s_0$ by the size of the model from MCP \citep{zhang2010nearly} that yields the lowest 10-fold cross validation (CV) MSE. Given that the following \lof algorithm can hardly handle hundreds of dimensions, we truncate our sparsity estimator below $100$. Denote the resulting sparsity estimator by $\wh s$. Then we use CoSaMP with projection size $\pi = \wh s$ and expansion size $l = \lceil\wh s / 2\rceil$ to generate a screening set $\wh \cS$ of size $\wh s$. To evaluate the quality of the screened set $\wh \cS$, in Figure \ref{fig:screen} we investigate the cardinality and true positive proportion (TPP) of $\wh \cS$, the latter of which is defined as
\begin{equation}\label{equ:def_tpp}
    \text{TPP}(\wh\cS) := \frac{|\wh\cS\cap \cS^0|}{|\cS^0|},
\end{equation}
where $\cS^0$ denotes the true support set. 
%Smaller model size indicates simpler task in the grouping stage.
% As mentioned in Section \ref{sec:ultra_met}, if TPP$(\wh{\cS})=1$, i.e., we achieve sure screening, the solution will remain unchanged under the sufficient condition while the dimensionality of the problem is much lower, thus significantly reducing the computational complexity.  As signal amplitude grows, the screened size decreases while the TPP gets
%improved. In most cases under $r\in\{0.2,0.25,0.3\}$, we achieve sure screening.
\begin{figure}
      \centering
      \begin{tabular}{c}
    \includegraphics[width = 8cm]{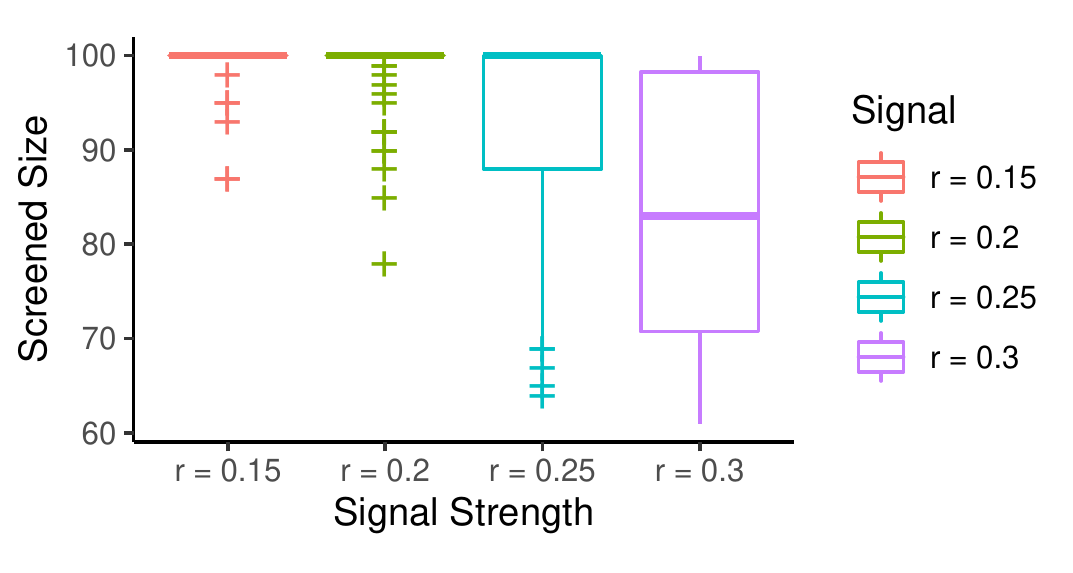} \\ \includegraphics[width = 8cm]{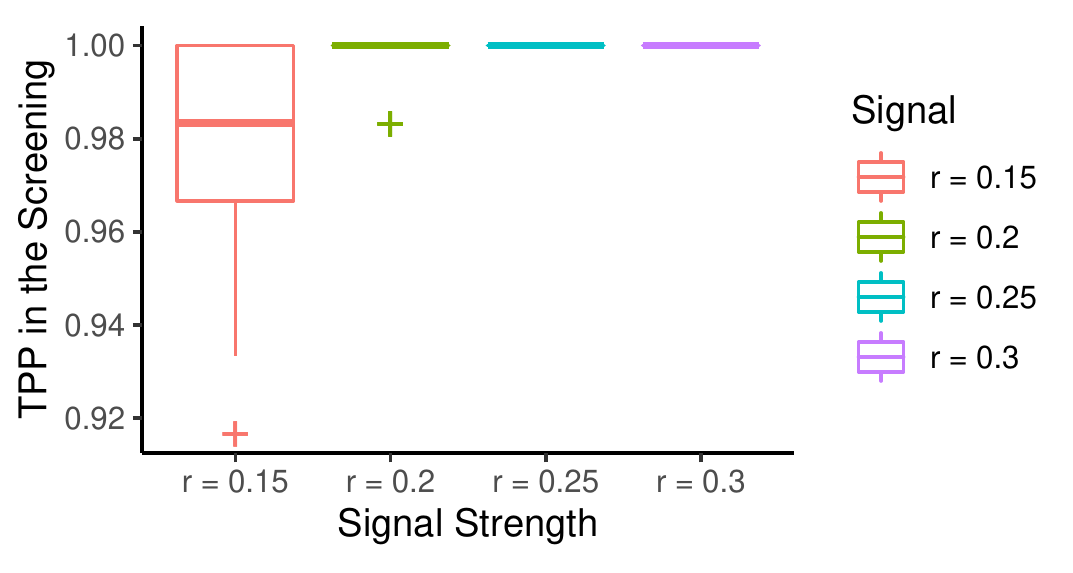}
       \end{tabular}
      \caption{Screening results for ultra-high dimensional problems. Here $p = 20,000$, $s_0 = 60$, $n=\lfloor 2s\log p\rfloor$.}
      \label{fig:screen}
\end{figure}
We then perform grouping on the reduced design. The implementations for all the grouping methods are similar as those in Section \ref{sec:nmi}. All the zero coefficients are considered as forming one group when we calculate the NMI. Figure \ref{fig:nmi_ultra} reports the NMI of \lof, CARDS, fLASSO, OLS+, pairReg+ and FGSG+.
%require a preliminary coefficient vector as an input or to construct graphs. This might be troublesome since the estimated coefficients after feature screening are actually biased. \lof, in contrast, does not suffer from this issue as discussed in Section \ref{sec:ultra_met}. 
We have the following observations from the two figures:
\begin{itemize}
    \item[(i)] Figure \ref{fig:screen} shows that as signal strength grows, CoSaMP yields higher TPP and smaller screening sizes, meaning that both  accuracy and efficiency of CoSaMP improve.
    \item[(ii)] Combining Figures \ref{fig:screen} and \ref{fig:nmi_ultra}, we observe that when $r\in\{0.2,0.25,0.3\}$, CoSaMP achieves sure screening while the grouping can be far from the truth. This suggests that the grouping error is due to the grouping stage rather than the screening stage.
    \item[(iii)] \lof still outperforms all the competing methods on the reduced designs.
\end{itemize}

\begin{figure}
      \centering
      \includegraphics[width = 14cm]{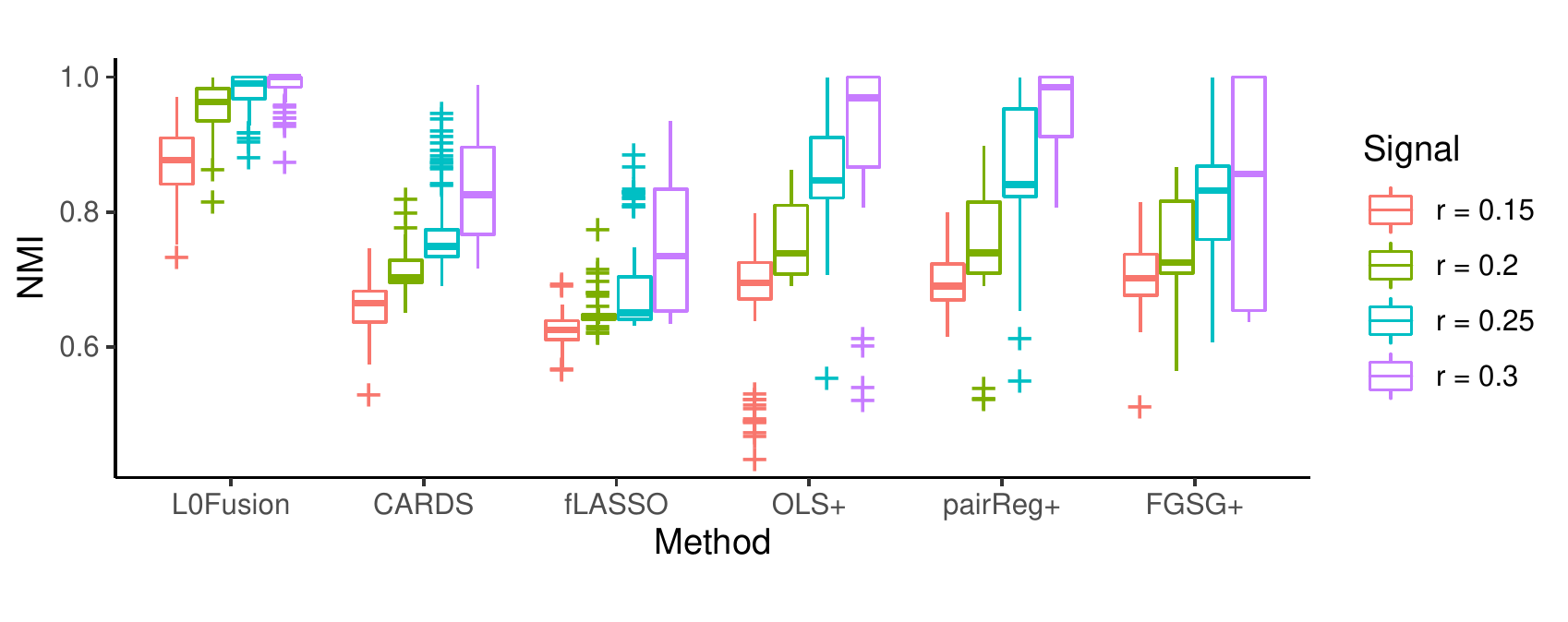} 
      \caption{Grouping results for ultra-high dimensional problems after $100$ Monte Carlo repetitions.}
      \label{fig:nmi_ultra}
\end{figure}

\subsection{Real data analysis}\label{sec:real}
We further illustrate the proposed method by an empirical study of metabolomics data collected from $n = 397$ adolescents consisting of $197$ boys and $200$ girls aged 8 to 18 years during a critical period of growth and sexuality maturation. Early onset of obesity in the adolescent years has been found to be associated with an increased risk of many diseases (e.g. hypertension, diabetics, and cancer) during adulthood. Thus, it is of great scientific interest to detect key groups of lipids (largest metabolites among metabolomics) that predict body mass index (BMI), adjusted by age and sex (1 for boy and 0 for girl). %Metabolites are essentially responsible for the supply of cellular energy via substances produced during metabolism such as digestion and other bodily chemical processes.
We investigate a total of $p = 234$ lipids to determine the number of lipid groups, group memberships and associated average contribution of a group predictor to BMI.  We fit the following linear model with a group homogeneity pursuit on the outcome of BMI:
\begin{equation}
    \label{eq:bmi}
\mbox{BMI} = \alpha_0 + \alpha_1 \mbox{age} + \alpha_2 \mbox{sex} +  \sum_{j=1}^{234} \beta_j \,\mbox{lip}_j + \varepsilon, ~\beta_j \in \{0, \gamma_1, \gamma_2, \dots, \gamma_K\}, ~\forall j \in [p],
\end{equation} 
where the number of signal groups ($K$) as well as  the group of null lipids with $\gamma_0 =0$ is determined by 10-fold CV. Each group represents a subset of lipids  with a shared nonzero effect size $\gamma_k, k=1,\ldots, K$. We normalize the design matrix to ensure mean $0$ and variance $1$ except intercept and sex. To calibrate the effect sizes with respect to different group sizes, for each group of lipids, we use their average measurement as the group's overall measurement; therefore, the corresponding group-level effect size is $\gamma_k$ multiplied by the group size.    

Here we adopt the aforementioned ``\textit{screen then group}'' strategy. Similarly to Section \ref{sec:ultra}, we first estimate the true sparsity $s$ by MCP with 10-fold CV and apply CoSaMP to identify promising individual lipids from the pool of 234 lipids. This screening step selects $18$ potential lipids together with intercept, age and gender. In the second phase, we perform \lof on these selected lipids. Through $10$-fold cross-validation over  $K=\{1, \ldots, 10\}$, we detect six groups with nonzero effect sizes.
% Figure \ref{fig:real_tune} gives the histograms of the chosen sparsity and group number respectively among $100$ Monte Carlo repetitions of random data splitting for CV. 
It is noteworthy that GUROBI solves the \lof problem within few seconds.  The results are summarized in Table \ref{tab:real}.  For group 1 consisting of 7 similar lipids, the group-level average lipid measurement has a 4.0 effect size on BMI. 

% \begin{figure}
%       \centering
%       \begin{tabular}{cc}
%       \includegraphics[width = 7cm]{fig_sparsity} & \includegraphics[width = 7cm]{K_tune}  
%       \end{tabular}
%       \caption{Histograms of tuned sparsity $s$ and tuned group number $K$. Random seed is set to be `2021' in R.}
%       \label{fig:real_tune}
% \end{figure}

%\iffalse 
%We report \textbf{average instead of total} for the grouping coefficients. Consider $\{X_j\}_{j=1}^{3}$ being in the same group with homogeneous coefficient $\beta_1$, then 
%\begin{equation}
%    \label{eq:bmi} 
%\end{equation}
%     \beta_1X_1 +\beta_1X_2+\beta_1X_3 = 3\beta_1\Bigl( \frac{X_1+X_2+X_3}{3} \Bigr) =: \tilde{\beta}_1 \Bigl( \frac{X_1+X_2+X_3}{3} \Bigr),
%\end{equation} 
%where $\tilde{\beta}_1$  is now measured for the average magnitude within this group. 
%\fi 
\begin{table}
	\centering
	\begin{tabular}{l  crr}
%	\begin{tabular}{l p{2cm} ccc}
		\hline
		\hline
		Group~(size)  & Features &\makecell[c]{Group-level \\ effect size}   \\
		\hline
         \makecell{} & Intercept & 21.88\\
         \makecell{} & Boys versus girls & -1.05 \\
         \makecell{} & Age & 0.62 \\
         \hline
         \makecell[c]{{Group 1}~{(7)}} & ``cholesterol biosynthesis"               & 4.00  \\
         \hline
         \makecell[c]{{Group 2}~{(6)}}& 
         ``nutritional energy support and regulation"
         & -3.13 \\
         \hline
         \makecell[c]{{Group 3}~{(2)}} & ``energy transport" & 4.44\\
         \hline
         \makecell[c]{{Group 4}~{(1)}} & ``diet signaling" & -2.00 \\
         \hline
         \makecell[c]{{Group 5}~{(1)}} & ``energy production" & -1.28\\
         \hline
         \makecell[c]{{Group 6}~{(1)}} & ``peptide hormones on food consumption"  &  0.99 \\
		\hline
	\end{tabular}
	\caption{The analysis results of lipid groups and effect sizes among $18$ promising metabolites from preliminary screening by CoSaMP.} 
	\label{tab:real}
\end{table}

%Given the fact that none of the 18 lipids belongs to the zero effect groups in the above analysis implies that the preliminary screening may be overly aggressive and select only strong predictors. To relax this first phase analysis, we plan to include more weak signals in the homogeneity pursuit analysis. Based on the solution paths generated by CoSaMP, we choose a subset of $85$ lipids in the analysis that is based on prefixed $K=6$. The results are reported in Tables \ref{tab:real87k6}, respectively. More details on the grouping structures are in the supplementary materials.

%\begin{table}[H]
%	\centering
%	\begin{tabular}{l  crr}
%	\begin{tabular}{l p{2cm} ccc}
%		\hline
%		\hline
%		\makecell[l]{{Group} {(size)}}  & Features & Effect size  \\
%		\hline
 %        \makecell{} & Intercept & 21.62\\
 %        \makecell{} & Boys versus girls & -0.55 \\
 %        \makecell{} & Age & 0.79 \\
 %        \hline
%         \makecell[l]{{Group 1} {(25)}} & name 1
%         & -8.64  \\
%         \hline
%         \makecell[l]{{Group 2} {(20)}} & name 2 & 7.28 \\
%         \hline
%        \makecell[l]{{Group 3} {(16)}} & 
%        name 3
%        & -11.18\\
%         \hline
%        \makecell[l]{{Group 4} {(15)}} & name 4 & 10.46 \\
%         \hline
%         \makecell[l]{{Group 5} {(6)}} & name 5 & 10.04\\
%         \hline
%         \makecell[l]{{Group 6}  {(3)}} & name 6 &  -7.27 \\
%		\hline
%	\end{tabular}
%	\caption{The analysis results of 6 lipid group and effect size among $85$ promising metabolites selected from a preliminary univariate screening.} 
%	\label{tab:real87k6}
%\end{table}

We also conduct a confirmatory semi-simulation using the metabolomics design matrix of this real dataset. Assume that a variable of interest $y$ relates to the metabolomics as follows:
\begin{equation}\label{equ:semi-simu}
y = \alpha_1 \mbox{age} + \alpha_2 \mbox{sex} +  \sum_{j=1}^{234} \beta_j \,\mbox{lip}_j + \varepsilon, ~\beta_j \in \{0, \gamma_1, \gamma_2, \dots, \gamma_K\}, ~\forall j \in [p].
\end{equation}
We set $\alpha_1 = 0.5$, $\alpha_2 = 1$ and randomly assign the coefficients $\bbeta$ with a sparse ($20$ nonzero values) and grouped ($4$ groups of size $5$) structure, where the true coefficients within each group are equal to $-2r,-r,r,2r$ respectively with $r\in\{0.5,1\}$. Then we generate $n=397$ responses from the metabolomics design according to \eqref{equ:semi-simu}, where $\epsilon$'s are i.i.d. $\cN(0,1)$. We assess the prediction performance with/without group structure along with the grouping accuracy by randomly splitting the observations into a training set and a testing set at each repetition. We then implement both the screening procedure alone and the \textit{screening then grouping} procedure on the training set and compare their prediction accuracy in terms of MSE on the testing set. 
% We use false discovery proportion (FDP) and true positive proportion (TPP) to evaluate the screening procedure. For any two subsets $\cS_1, \cS_2 \subset [p]$, let $\cS_1 \backslash \cS_2$ denote $\cS_1 \cap \cS_2 ^ c$. For an estimated active set $\wh{\cS}$ of the true support $\cS^0$, the definition of TPP is given in \eqref{equ:def_tpp} and the FDP is defined as
% \[
%     \text{FDP}(\wh\cS) := \frac{|\wh \cS \backslash \cS ^ {0}|}{\max(|\wh\cS|, 1)}. 
% \] 
% We report the mean and standard error of FDR and TPP in $100$ Monte Carlo repetitions. 
% To assess the prediction error, we use mean squared error (MSE), recording the mean and standard error in $100$ Monte Carlo repetitions.
Table \ref{tab:91} reports the testing MSE with standard error as well as the quantiles of NMI based on $100$ independent Monte Carlo repetitions. It is clear from the table that leveraging the existing group structure by \lof improves the prediction accuracy. 

\begin{table}
	\centering
	\begin{tabular}{c|c|c|c}
%	\begin{tabular}{l p{2cm} ccc}
		\hline
		\hline
		\makecell[c]{Signal\\ Strength }  & \makecell[c]{Prediction Error \\ (with grouping)} & \makecell[c]{Prediction Error \\ (without grouping)} &
		\makecell[c]{NMI \\ (quantiles)} \\
		\hline
         \makecell{$r = 1$} & \makecell[c]{\textbf{1.092(0.031)}} & \makecell[c]{1.186(0.029)} 
         &
         \makecell[l]{ Max: 1.00 \\  3rd Qu: 1.00 \\  Median: 1.00\\  1st Qu:1.00 \\  Min: 0.75 }\\
         \hline
         \makecell{$r = 0.5$} & \makecell[c]{\textbf{1.120(0.023)}} & \makecell[c]{1.223(0.001)} 
         &
         \makecell[l]{ Max: 1.00 \\  3rd Qu: 1.00 \\  Median: 0.91\\  1st Qu: 0.85\\  Min: 0.58 } \\
         \hline
	\end{tabular}
	\caption{
%	In this experiment, true sparsity is set to be $20$. The true predictors are divided into $4$ groups of size $(5,5,5,5)$.  True coefficients within each group are $(-1, -0.5,0.5,1 )$ respectively. In each Monte Carlo repetition, we randomly choose $20$ metabolites as the true signals and generate the responses. He we split the data equally into $10$ folds. $9$ folds for training and $1$ fold for testing. We report the mean (standard error) in this table. It is clear that prediction with grouping outperforms prediction without grouping.
    Results of grouping and prediction accuracy for the semi-simulation.} 
	\label{tab:91}
\end{table}

\section{Discussion}

This paper studies a combinatorial approach called \lof that enables simultaneous operation of clustering and estimation for regression coefficients in a linear model. This analytic task addresses a practical need for learning  homogeneous groups of nonzero regression coefficients in a regression analysis to assess the relationship between outcomes and clustered signal features. We propose to formulate the \lof problem as a mixed integer optimization (MIO) problem and then leverage modern MIO solvers to compute the corresponding estimators. When the dimension is too high for the MIO solver to handle, we invoke CoSaMP as a preliminary variable screening procedure to reduce dimension prior to the $L_0$-Fusion.  As shown theoretically and numerically in Section \ref{sec:ultra}, such a ``screen then group'' strategy dramatically broadens the applicability of the homogeneity fusion technique, which can scale \lof up to the ambient dimension $p = 20,000$ with high accuracy of recovering the true group structure of regression coefficients. This level of methodological capacity allows to handle a large number of modern biomedical datasets. Thus, this two-stage approach as well as its variants provide efficient toolboxes to solve many homogeneity fusion problems on large-scale datasets. 

Theoretically, we establish grouping consistency of the \lof estimator, for which the sample size $n$ only needs to grow at the same rate as the sum of logarithms of the true sparsity and true group number, i.e., $\log(pK)$. This sample size requirement is also shown to be necessary for any procedure to achieve grouping consistency. These technical results are not only of theoretical interest, but also useful to guide practical work such as sample size determination in a study design. 

% The proposed upper bound in the search of warm start may be generalized to a general convex objective function with little effort. {\color{red} However, generalization of the lower bound to a setting beyond the least square objective function is not that straightforward due to the complexity of the kernel function used in the formulation.} This is worth a future exploration.  Also, due to the use of the ridge penalty in the formation of the lower bound problem, the resulting lower bound solution is just an approximate to the original optimization problem, but this approximation can be asymptotically diminished if the tuning parameter is chosen in the order of $o(1/n)$, under which the clustering consistency is warrant. 

An important future work concerns statistical inference after the operation of $L_0$-Fusion. A thorough investigation on the influence of selection errors on statistical inference, in both aspects of finite-sample and large-sample properties, is of great interest. This \lof may be extended to other regression problems with the framework of generalized linear models where iterative procedures used in the parameter estimation rely on weighted least squares objective functions. Thus, this extension is technically manageable but may require substantial computational effort. Also, we would consider an extension of this method to the setting of estimating equations, which could cover a broad range of important statistical models, such as GEE regression, Cox regression and quantile regression. 

\section*{Acknowledgements}
This research is supported by a National Institutes of Health grant R01ES024732 and two National Science Foundation grants DMS2113564 and DMS2015366. We are grateful to
Dr. Zheng Tracy Ke and Dr. Xiaotong Shen for providing their R codes for CARDS and FGSG methods.

\renewcommand{\thethm}{{\sc A.\arabic{theorem}}}
\renewcommand{\thelemma}{A.\arabic{lemma}}\setcounter{lemma}{0}
\renewcommand{\thesection}{{A.\arabic{section}}}

\renewcommand{\theequation}{{A.\arabic{equation}}}
\renewcommand{\thesubsection}{{\it A.\arabic{subsection}}}
\setcounter{equation}{0}
\setcounter{subsection}{0}

\newpage
\section*{Appendices}
\subsection{Segment neighborhood method}
\label{subsection:appendixAlg}

Algorithm \ref{alg:0} is a generalization of the segment neighbourhood method \citep{auger1989algorithms} with sparsity constraint.

\begin{center}
    \begin{algorithm}[h]
        \caption{Segment neighborhood method $\Omega(Ks^2+p\log p+q)$ }
        \label{alg:0}
        \KwIn{  $\bc\in\mathbb{R}^{q+p}$, the number of groups $K$ and the sparsity restriction $s$ }
        \KwOut{a member in $H_{K,s}(c)$}
        \begin{algorithmic}[1]
            \State $\hat\balpha=(c_{1}, c_{2}, \hdots, c_{q})^\top$.
            \State  Let $\delta$ be a bijection on $\{q+1,\hdots,q+p\}$ such that $c_{\delta(q+1)}\leq c_{\delta(q+2)}\leq\hdots\leq c_{\delta(q+p)}$.
            \State  set $x_{lk},y_{lk},x'_{lk}$ and $y'_{lk}$ to $0$ for $l=0,\hdots,p+1$ and $k=0,\hdots,K$.
            \State For $l$ from $1$ to $s$: \\
					\forceindent\forceindent For $k$ from $1$ to $K$: \\
					\forceindent\forceindent\forceindent\forceindent $x_{lk}=\argmax_{1 \leq i\leq l}\{y_{i-1,k-1}+\frac{(\sum_{j=i}^l c_{\delta(q+j)})^2}{l-i+1}\}$.\\
					\forceindent\forceindent\forceindent\forceindent $y_{lk}=\max_{1 \leq i\leq l}\{y_{i-1,k-1}+\frac{(\sum_{j=i}^l c_{\delta(q+j)})^2}{l-i+1}\}$.
			\State 	If $s<p$:\\
					\forceindent\forceindent For $l$ from $p$ to $p-s+1$: \\
					\forceindent\forceindent\forceindent\forceindent For $k$ from $1$ to $K$: \\
					\forceindent\forceindent\forceindent\forceindent\forceindent\forceindent $x'_{lk}=\argmax_{l \leq i\leq p}\{y'_{i+1,k-1}+\frac{(\sum_{j=l}^i c_{\delta(q+j)})^2}{i-l+1}\}$.\\
					\forceindent\forceindent\forceindent\forceindent\forceindent\forceindent $y'_{lk}=\max_{l \leq i\leq p}\{y'_{i+1,k-1}+\frac{(\sum_{j=l}^i c_{\delta(q+j)})^2}{i-l+1}\}$.\\
			\State 	$
					(k^*,l^*,m^*)=\argmax_{\substack{0\leq k\leq K\\0\leq l\leq s\\p-min(s,p)+l+1\leq m \leq p }} y_{kl}+y'_{K-k,m}
				$
			\State For $l$ from $l^*+1$ to $m^*-1$:\\
				\forceindent\forceindent $\hat\beta_{\delta(q+l)}=0$.
			\State Set $t=l^*$\\
				For $k$ from $k^*$ to $1$:\\
				\forceindent\forceindent For $l$ from $t$ to $x_{t,k}$:\\
				\forceindent\forceindent\forceindent\forceindent $\hat\beta_{\delta(q+l)}=\frac{\sum_{j=t}^{x_{tk}}c_{\delta(q+j)}}{t-x_{tk}+1}$.\\
				\forceindent\forceindent $t=x_{tk}-1$.
			\State 	Set $t=m^*$\\
				For $k$ from $K-k^*$ to $1$:\\
				\forceindent\forceindent For $l$ from $t$ to $x'_{tk}$:\\
				\forceindent\forceindent\forceindent\forceindent $\hat\beta_{\delta(q+l)}=\frac{\sum_{j=t}^{x'_{tk}}c_{\delta(q+j)}}{x'_{tk}-t+1}$.\\
				\forceindent\forceindent $t=x'_{tk}+1$.
			\State Return $(\hat\balpha^\top,\hat\bbeta^\top)^\top$.
        \end{algorithmic}
    \end{algorithm}
\end{center}

\subsection{Proofs of Propositions and Theorems}
\subsubsection{Proof of Proposition~\ref{prop:4}}
\label{app:1}

 \begin{prop}
 	\label{prop:2}
 	Consider problem~\eqref{eq:problemws} and some constant $L>l$. Let $\{\btheta_m\}_{m=1}^\infty$ be the sequence generated by Algorithm~\ref{alg:warm_start}. Define 
 	\begin{eqnarray*}
% 	\label{eq:rho}
 	&&\rho_m = \left\{ \begin{array}{ll}
 	\min_{\substack{\beta_{m, j} \neq \beta_{m, j'}\\ {\beta_{m, j}, \beta_{m, j'} \neq 0}}}|\beta_{m, j} - \beta_{m, j'}|, & \text{if there are $K$ distinct non-zero values in $\bbeta_m$},\\
 	0,& otherwise;
 	\end{array}\right. \nonumber\\
 	&&\tau_m = \left\{ \begin{array}{ll}
 	\min_{\beta_{m, j} \neq 0} |\beta_{m, j}|, & \text{if there is some non-zero value in $\bbeta_{m}$}, \\
 	0,  & \text{otherwise}.
 	\end{array} 	
 	\right.  	
 	\end{eqnarray*}
 	The following properties hold.    
%If $\lim\inf_{m \to \infty} \tau_m > 0$ and there exists a convergent subsequence $\btheta_{m_s}$ such that $\lim_{m\to \infty}\rho_{m_s} = 0$, then $\btheta_{m_s}$ converges to a first-order stationary point, where
	\begin{enumerate}
		\item 
			When $\liminf_{m\to\infty}\rho_m>0$ and $\liminf_{m\to\infty}\tau_m>0$, we have
			\begin{enumerate}
				\item 
					$\mathbb{G}(\bbeta_m)$ converges. \label{statement:converge1}
				\item 
					If $g$ has second order derivative and there exists $l'>0$ such that $l'\left\Vert \btheta-\tilde\btheta \right\Vert_2\leq \left\Vert \nabla g(\btheta)-\nabla g(\tilde\btheta) \right\Vert_2$ for any $\btheta,\tilde\btheta\in\bTheta(K,s)$ satisfying $\mathbb{G}(\bbeta)=\mathbb{G}(\tilde\bbeta)$, then the sequence $\btheta_m$ is bounded and converges to a first-order stationary point.\label{statement:converge2}
			\end{enumerate}
		\item 
			When $\liminf_{m\to\infty}\tau_m=0$, we have
			\begin{enumerate}
				\item 
					$\liminf_{m\to\infty}\left\Vert\nabla g(\btheta_m)\right\Vert_\infty=0$. \label{statement:converge3}
				\item 
					If there exists a convergent subsequence $\{\btheta_{f(m)}\}_{m=1}^\infty$ such that $\lim_{m\to\infty}\tau_{f(m)}=0$, then $\lim_{m\to\infty}g(\btheta_m)=\min_{\btheta\in\mathbb{R}^{p+q}}g(\btheta)$.\label{statement:converge4}
			\end{enumerate}
		\item 
			When $\liminf_{m\to\infty}\rho_m=0$ and $\liminf_{m\to\infty}\tau_m>0$, we have 
			\begin{enumerate}
				\item 
					$\mcG(\bbeta_m;0)$ converges and $\liminf_{m\to\infty}\max_{j\in [p]\backslash\mcG(\btheta_m;0)}\abs{\frac{\partial g(\btheta_m)}{\partial\beta_j}}
					=0$.\label{statement:converge5}
				\item 
					If there exists a convergent subsequence $\{\btheta_{f(m)}\}_{m=1}^\infty$ such that $\lim_{m\to\infty}\rho_{f(m)}=0$, then $\btheta_{f(m)}$ converges to a first-order stationary point.\label{statement:converge6}
			\end{enumerate}
	\end{enumerate}
 \end{prop} 

	\begin{rem}
		\label{remark:1}
		The convergent subsequence condition could be satisfied under some weak conditions such as $\{\btheta \in \bTheta(K, s) |~ g(\btheta) \leq C\}$ being bounded for any $C\in\mathbb{R}$.
	\end{rem}

	\begin{proof}
	In the following proof,  $\beta_{m,j}$ denotes the $j$th element in $\bbeta_m$  where $\btheta_m=(\bbeta_m^\top,\balpha_m^\top)^\top$. Likewise, $c_{m,j}$ denotes the $j$th element in $\bc_m:=  \btheta_{m}-\frac{1}{L}\nabla g(\btheta_m)$. 
		\begin{enumerate}
			\item
				\begin{enumerate}
					\item
						%Since $\liminf_{m\to\infty}\rho_m>0$ and $\liminf_{m\to\infty}\tau_m>0$, there exists $M'\in\mathbb{Z}^+$ and $a>0$ such that for any $m\geq M'$ we have $\rho_m>a$ and $\tau_m>a$. If $\mathbb{G}(\bbeta_m)$ does not converge, then for any $m\in\mathbb{Z}^+$, there exists $M\geq \max(m,M')$ such that $\left\Vert\theta_{M+1}-\theta_{M}\right\Vert_2>a/\sqrt{2}$, in contradiction to Proposition~\ref{prop:3} Statement~\ref{statement:shrink}.     
						For large enough $m$, if $\mathbb{G}(\bbeta_m)\neq\mathbb{G}(\bbeta_{m+1})$, then 
						\[
						\left\Vert\bbeta_{m+1}-\bbeta_{m}\right\Vert_2>\min(\liminf_{m\to\infty}\rho_m,\liminf_{m\to\infty}\tau_m)/\sqrt{2},
						\]
						in contradiction to Proposition~\ref{prop:3} Statement~\ref{statement:shrink}. 
					\item
						Due to Statement~\ref{statement:converge1}, there exists $M$ such that for any $m\geq M$, $\mathbb{G}(\bbeta_m)$ are the same. Then for any $m>M$, we have
						\begin{align*}
							&\left\Vert \btheta_{m+2}-\btheta_{m+1} \right\Vert_2 
							=\left\Vert 
							\begin{pmatrix}
								\textbf{A}&0\\
								0&\textbf{I}_q
							\end{pmatrix}
							\brac[\Big]{(\btheta_{m+1}-\btheta_{m})-\frac{1}{L}
						\paren[\Big]{\nabla g(\btheta_{m+1})-\nabla g(\btheta_{m})}}\right\Vert_2\\
				=&\left\Vert
						\begin{pmatrix}
							\textbf{A}&0\\
							0&\textbf{I}_q
						\end{pmatrix}
						(I-\frac{1}{L}\nabla^2 g(\btheta'))
						(\btheta_{m+1}-\btheta_{m})\right\Vert_2
						\leq\sqrt{1-\frac{l'^2}{L^2}}\left\Vert \btheta_{m+1}-\btheta_{m}\right\Vert_2,
				\end{align*}
				where $\textbf{A}_{p\times p}$ is an idempotent matrix $\left(\frac{I(\beta_{m,j}=\beta_{m,j'})}{|\mathbb{G}(\bbeta_m;\beta_{m,j'})|}\right)_{j,j'\in[p]}$.
			Since $0<\frac{l'}{L}\leq 1$, $\btheta_m$ converges to a first order stationary point.
				\end{enumerate}
			\item
				\begin{enumerate}
						\item
							Since $\btheta_{m+1}-\btheta_{m}$ converges, we have $\lim_{m\to\infty}\left\Vert \frac{\partial g(\btheta_m)}{\partial \balpha}\right\Vert_\infty=0$. There exists a subsequence $\{\btheta_{f(m)}\}_{m=1}^\infty$ such that $\lim_{m\to\infty}\tau_{f(m)}=0$. Without loss of generality, we assume $|\mcG(\bbeta_{f(m)};\tau_{f(m)})|=t>0$. Fixing $m$, for any $j\in[p]$ such that $|\mcG(\bbeta_{f(m)};\beta_{f(m),j})|=t'>1$, we create $\tilde\btheta$ whose grouping is the same as $\btheta_{f(m)}$ except that the $0$-group and $\tau_{f(m)}$-group in $\btheta_{f(m)}$ are merged as the new $0$-group and that $\beta_{f(m),j}$ is singled out as a new group. Then 
						\begin{align*}
							&0\geq \frac{2}{L}\brac[\big]{h_L(\btheta_{f(m)},\btheta_{f(m)-1})-h_L(\tilde\btheta,\btheta_{f(m)-1})}\\
							&=\begin{cases}
							-t\tau^2_{f(m)}+c_{f(m)-1,j}^2, \hfill\text{ if }\beta_{f(m),j}=0 \text{ or }\tau_{f(m)};\\
    						-t\tau^2_{f(m)}+\frac{t'}{t'-1}(\beta_{f(m),j}-c_{f(m)-1,j})^2, \hfill\text{ otherwise}.
						    \end{cases}
						\end{align*}
						So for any $j\in [p]$, we have $\frac{1}{L}|\frac{\partial g(\theta_{f(m)-1})}{\partial \beta_{j}}|=|\beta_{f(m)-1,j}-c_{f(m)-1,j}|\leq|\beta_{f(m)-1,j}-\beta_{f(m),j}|+|\beta_{f(m),j}-c_{f(m)-1,j}|\leq \left\Vert \btheta_{f(m)}-\btheta_{f(m)-1} \right\Vert_2+(\sqrt{s}+1)\tau_{f(m)}$. Thus $\lim_{m\to\infty}\left\Vert \frac{\partial g(\btheta_{f(m)-1})}{\partial \bbeta}\right\Vert_\infty=0$. 
					\item 
						Due to Statement~\ref{statement:converge3}, we have $\lim_{m\to\infty}\left\Vert\nabla g(\btheta_{f(m)-1})\right\Vert_\infty=0$. Since $\lim_{m\to\infty}\btheta_{f(m)-1}=\btheta'$, we have $g(\btheta')=\min_{\btheta} g(\btheta)$. Since $g(\btheta_m)$ converges, we have $\lim_{m\to\infty}g(\btheta_m)=\min_{\btheta} g(\btheta)$.
				\end{enumerate}
			\item
				\begin{enumerate}
					\item 
						Due to the proof of Statement~\ref{statement:converge1}, if $\liminf_{m\to\infty}\rho_m=0 $, then $\mcG(\bbeta_m;0)$ converges. There exists sequences $\{\btheta_{f(m)}\}_{m=1}^\infty$, $\{j_m\}_{m=1}^\infty$ and $\{j'_m\}_{m=1}^\infty$ such that for any $m>0$ we have $\beta_{f(m),j_m}\neq 0$, $\beta_{f(m),j'_m}\neq 0$, $\abs{\beta_{f(m),j_m}-\beta_{f(m),j'_m}}=\rho_{f(m)}$ and $\lim_{m\to\infty}\beta_{f(m),j_m}-\beta_{f(m),j'_m}=0$. Fixing $m$, let $t$ and $t'$ denote $|\mcG(\bbeta_{f(m)};\beta_{f(m),j_m})|$ and $|\mcG(\bbeta_{f(m)};\beta_{f(m),j'_m})|$. For any $j''\in[p]$ such that $\beta_{f(m),j''}\neq 0$ and $t'':=|\mcG (\bbeta_{f(m)};\beta_{f(m),j''})|>1$, we create $\tilde\btheta$ whose grouping is the same as $\btheta_{f(m)}$ except that the $\beta_{f(m),j_m}$-group and $\beta_{f(m),j'_m}$-group in $\btheta_{f(m)}$ are merged as a new group and that $\beta_{f(m),j''}$ is singled out as a new group. Then 
						\begin{align*}
							&0\geq \frac{2}{L}\brac[\big]{h_L(\btheta_{f(m)},\btheta_{f(m)-1})-h_L(\tilde\btheta,\btheta_{f(m)-1})}\\
							=&\begin{cases}
							-(\frac{1}{t}+\frac{1}{t'})^{-1}(\beta_{f(m),j_m}-\beta_{f(m),j'_m})^2+\frac{t+t'}{t+t'-1}(\frac{t\beta_{f(m),j_m}+t'\beta_{f(m),j'_m}}{t+t'}-c_{f(m)-1,j''})^2,\\
							\hfill\text{if }\beta_{f(m),j''}=\beta_{f(m),j_m} \text{ or } \beta_{f(m),j'_m};\\
							\\
							-(\frac{1}{t}+\frac{1}{t'})^{-1}(\beta_{f(m),j_m}-\beta_{f(m),j'_m})^2+\frac{t''}{t''-1}(\beta_{f(m),j''}-c_{f(m)-1,j''})^2,\hfill\text{otherwise}.
						\end{cases}
					\end{align*}
					So for any $j''\in[p]$ such that $\beta_{f(m),j''}\neq 0$, we have $\frac{1}{L}|\frac{\partial g(\btheta_{f(m)-1})}{\partial \beta_{j''}}|=|\beta_{f(m)-1,j''}-c_{f(m)-1,j''}|\leq|\beta_{f(m)-1,j''}-\beta_{f(m),j''}|+|\beta_{f(m),j''}-c_{f(m)-1,j''}|\leq \left\Vert\btheta_{f(m)-1}-\btheta_{f(m)}\right\Vert_2+(s+1)\rho_{f(m)}$. 
				\item
					On top of the proof of Statement~\ref{statement:converge5}, for fixed $m$ and any $j',j''\in[p]$ such that $\beta_{f(m),j'}=0$ and $\beta_{f(m),j''}\neq 0$, we create $\tilde\btheta$ whose grouping is the same with $\btheta_{f(m)}$ except that the $\beta_{f(m),j_m}$-group and $\beta_{f(m),j'_m}$-group in $\btheta_{f(m)}$ are merged as a new group and that $\beta_{j'}$ is singled out as a new non-zero group and $\beta_{j''}$ is put in $0$-group. Let $t''$ denote $|\mcG(\bbeta_{f(m)};\beta_{f(m),j''})|$. Then 
					\begin{align*}
						&0\geq \frac{2}{L}\brac[\big]{h_L(\btheta_{f(m)},\btheta_{f(m)-1})-h_L(\tilde\btheta,\btheta_{f(m)-1})}\\
						=&\begin{cases}
						-(\frac{1}{t}+\frac{1}{t'})^{-1}(\beta_{f(m),j_m}-\beta_{f(m),j'_m})^2+\frac{t+t'}{t+t'-1}(\frac{t\beta_{f(m),j_m}+t'\beta_{f(m),j'_m}}{t+t'}-c_{f(m)-1,j''})^2\\
						-c^2_{f(m)-1,j''}+c^2_{f(m)-1,j'}, \hfill\text{if }\beta_{f(m),j''}=\beta_{f(m),j_m} \text{ or } \beta_{f(m),j'_m};\\
						\\
						-(\frac{1}{t}+\frac{1}{t'})^{-1}(\beta_{f(m),j_m}-\beta_{f(m),j'_m})^2+\frac{I(t''>1)t''}{t''-1}(\beta_{f(m),j''}-c_{f(m)-1,j''})^2\\
						-c^2_{f(m)-1,j''}+c^2_{f(m)-1,j'},\hfill\text{\ \ \ otherwise}.
					\end{cases}
					\end{align*}
					Thus $\abs{c_{f(m)-1,j'}}\leq \abs{c_{f(m)-1,j''}}$.\\
					
					Since $\lim_{m\to\infty}\btheta_{f(m)}=\btheta'$, we have $\lim_{m\to\infty}\btheta_{f(m)-1}=\btheta'$ and $\lim_{m\to\infty}\nabla g(\btheta_{f(m)-1})=\nabla g(\btheta')$. It is easy to check that $\bTheta(K,s)$ is a closed set, so $\btheta'\in\bTheta(K,s)$. And $\mcG(\bbeta';0)=\lim_{m\to\infty}\mcG(\bbeta_m;0)$ because of Statement~\ref{statement:converge5}. Therefore, we have \[
					    \abs{\beta'_{j'}-\frac{1}{L}\frac{\partial g(\btheta')}{\partial \beta_{j'}}}=\lim_{m\to\infty}\abs{c_{f(m)-1,j'}}\leq \lim_{m\to\infty}\abs{c_{f(m)-1,j}}=\abs{\beta'_{j}-\frac{1}{L}\frac{\partial g(\btheta')}{\partial \beta_{j}}},
					\]
					for any $j,j'\in[p]$ such that $\beta'_j=0$ and $\beta'_{j'}\neq 0$.  Due to  Statement~\ref{statement:converge5}, we have $\frac{\partial g(\btheta')}{\partial \beta_j}=0 $ for any $j\in[p]$ such that $\beta'_j=0$. So $\btheta'\in\cH_{K,s}(\btheta'-\frac{1}{L}\nabla g(\btheta'))$.
		\end{enumerate}
\end{enumerate}
	\end{proof}

\subsubsection{Proof of Theorem~\ref{thm:suff}}
\label{app:3}
	\begin{proof}
		For any grouping $\mathbb{G}(\btheta)$ such that  $\btheta\in\bTheta(K_0,s_0)$, define $\bP_{\mathbb{G}(\btheta)}$ as the projection matrix of $(\bX_{\mathbb{G}(\btheta)},\bZ)$. For any $\btheta$ satisfying $\btheta\in\bTheta(K_0,s_0)$ and $\mathbb{G}(\bbeta)\neq \mathbb{G}(\bbeta ^ *)$, we have: 
		\begin{align}
			&\mbP\paren[\Bigg]{
			\min_{\substack{\tilde\btheta\in\bTheta(K_0,s_0)\\ \mathbb{G}(\tilde\bbeta)=\mathbb{G}(\bbeta)}}\norm{ \bY-(\bX,\bZ)\tilde\btheta}_2^2<\norm{ \bY-(\bX,\bZ)\hat\btheta^{\mathrm{ol}}}_2^2}\nonumber\\
			=&\mbP\left(2\bvarepsilon^\top(\bI-\bP_{\mathbb{G}(\bbeta)})(\bX,\bZ)\btheta^* + \bigl\Vert(\bI-\bP_{\mathbb{G}(\bbeta)})(\bX,\bZ)\btheta^*\bigr\Vert_2^2-\bvarepsilon^\top(\bP_{\mathbb{G}(\bbeta)}-\bP_{\mathbb{G}(\bbeta ^ *)})\bvarepsilon<0\right). \label{Equation:suff:prob1}
		\end{align}
		For any $0<\delta<1$, we have: 
		\begin{align*}
			\text{Equation~(\ref{Equation:suff:prob1})}\leq&\mbP\left(2\bvarepsilon^\top(\bI-\bP_{\mathbb{G}(\bbeta)})(\bX,\bZ)\btheta^*+\delta\bigl\Vert(\bI-\bP_{\mathbb{G}(\bbeta)})(\bX,\bZ)\btheta^*\bigr\Vert_2^2<0\right)+\\
			&\mbP\left((1-\delta)\bigl\Vert(\bI-\bP_{\mathbb{G}(\bbeta)})(\bX,\bZ)\btheta^*\bigr\Vert_2^2-\bvarepsilon^\top(\bP_{\mathbb{G}(\bbeta)}-\bP_{\mathbb{G}(\bbeta ^ *)})\bvarepsilon<0\right).
		\end{align*}
		For any $t_1>0$, $t_2>0$ and by Markov's inequality, we have:
		\begin{align}
			\text{Equation~(\ref{Equation:suff:prob1})}\leq&\mbE\left[\exp{\brac[\bigg]{-\frac{2t_1 \bvarepsilon^\top(\bI-\bP_{\mathbb{G}(\bbeta)})(\bX,\bZ)\btheta^*}{\sigma^2}}}\right]\exp{\brac[\Bigg]{-\frac{t_1\delta\bigl\Vert(\bI-\bP_{\mathbb{G}(\bbeta)})(\bX,\bZ)\btheta^*\bigr\Vert_2^2}{\sigma^2}}}+\label{Equation:suff:term1}\\
			&\mbE\left[\exp{\brac[\bigg]{\frac{t_2\bvarepsilon^\top(\bP_{\mathbb{G}(\bbeta)}-\bP_{\mathbb{G}(\bbeta ^ *)})\bvarepsilon}{\sigma^2}}}\right]\exp{\brac[\Bigg]{-\frac{t_2(1-\delta)\bigl\Vert(\bI-\bP_{\mathbb{G}(\bbeta)})(\bX,\bZ)\btheta^*\bigr\Vert_2^2}{\sigma^2} }}.\label{Equation:suff:term2}
		\end{align}
%			\leq&\exp{\brac[\Bigg]{\frac{2 t_1^2\bigl\Vert(\bI-\bP_{\mathbb{G}(\bbeta)})(\bX,\bZ)\btheta_0\bigr\Vert_2^2}{\sigma^2}}}\exp{\brac[\Bigg]{-\frac{t_1\delta\bigl\Vert(\bI-\bP_{\mathbb{G}(\bbeta)})(\bX,\bZ)\btheta_0\bigr\Vert_2^2}{\sigma^2}}}+\\
%			&\mbE\left[\exp{\brac[\bigg]{\frac{t_2\bvarepsilon^\top(\bP_{\mathbb{G}(\bbeta)}-\bP_{\mathbb{G}(\bbeta_0)})\bvarepsilon}{\sigma^2}}}\right]\exp{\brac[\Bigg]{-\frac{t_2(1-\delta)\bigl\Vert(\bI-\bP_{\mathbb{G}(\bbeta)})(\bX,\bZ)\btheta_0\bigr\Vert_2^2}{\sigma^2} }},\\
%			&\text{by moment generating function of multivariate normal distribution, }\\
%			\leq& \exp\brac[\bigg]{\frac{2t_1^2-t_1\delta}{\sigma^2}n d(\bbeta,\bbeta_0) c_{\min}}+\\ 
%			&\exp\brac[\bigg]{-\frac{(1-\delta)t_2 n d(\bbeta,\bbeta_0) c_{\min}}{\sigma^2}+ 2t_2\abs{\mathbb{G}(\bbeta)\backslash\mathbb{G}(\bbeta_0)}},\\ 
%			&\text{when $2t_1<\delta$ due to Lemma~4 in \cite{shen2013constrained} and proof of Theorem 2 in \cite{shen2013constrained}}\\
%			\leq& 2\exp\brac[\bigg]{-\frac{n d(\bbeta,\bbeta_0) c_{\min}}{18\sigma^2}+\frac{2}{3} |\mathbb{G}(\bbeta)\backslash\mathbb{G}(\bbeta_0)|}, \text{ when $t_1=t_2=\frac{1}{3}$ and $\delta=\frac{5}{6}$}.
            By moment generating function, when $2t_1^2-t_1\delta<0$, the term in Equation~(\ref{Equation:suff:term1}) equals to:  
            \begin{align*}
            	&\exp{\brac[\Bigg]{\frac{2 t_1^2\bigl\Vert(\bI-\bP_{\mathbb{G}(\bbeta)})(\bX,\bZ)\btheta^*\bigr\Vert_2^2}{\sigma^2}}}\exp{\brac[\Bigg]{-\frac{t_1\delta\bigl\Vert(\bI-\bP_{\mathbb{G}(\bbeta)})(\bX,\bZ)\btheta^*\bigr\Vert_2^2}{\sigma^2}}},\\
            	\leq&\exp\brac[\bigg]{\frac{2t_1^2-t_1\delta}{\sigma^2}n d(\bbeta,\bbeta ^ *) c_{\min}}.
            \end{align*}
            By geometry interpretation of projection matrix, the term in Equation~(\ref{Equation:suff:term2}) is smaller than or equal to:    
            \begin{align*}
            	&\mbE\left[\exp{\brac[\bigg]{\frac{t_2\bvarepsilon^\top\bP_{\mathbb{G}(\bbeta)\backslash\mathbb{G}(\bbeta ^ *)}\bvarepsilon}{\sigma^2}}}\right]\exp{\brac[\Bigg]{-\frac{t_2(1-\delta)\bigl\Vert(\bI-\bP_{\mathbb{G}(\bbeta)})(\bX,\bZ)\btheta^*\bigr\Vert_2^2}{\sigma^2}}},\\
            	&=(1-2t_2)^{-\abs{\mathbb{G}(\bbeta)\backslash\mathbb{G}(\bbeta ^ *)}/2}\exp{\brac[\Bigg]{-\frac{t_2(1-\delta)\bigl\Vert(\bI-\bP_{\mathbb{G}(\bbeta)})(\bX,\bZ)\btheta^*\bigr\Vert_2^2}{\sigma^2}}},
            \end{align*}
            where $\bP_{\mathbb{G}(\bbeta)\backslash\mathbb{G}(\bbeta ^ *)}$ indicates the projection matrix of the columns in $X_{\mathbb{G}(\bbeta)}$ but not in $X_{\mathbb{G}(\bbeta ^ *)}$, and $\abs{\mathbb{G}(\bbeta)\backslash\mathbb{G}(\bbeta ^ *)}$ is the number of those columns.
            By the fact that $2t_2\geq -\log(1-2t_2)/2$ for any $0<t_2<0.398$, we can restrict $t_2\leq0.398$. Then, the term in Equation~(\ref{Equation:suff:term2}) is less than or equal to: 
            \begin{align*}
            	&
            	\exp{\brac[\Bigg]{2t_2\abs{\mathbb{G}(\bbeta)\backslash\mathbb{G}(\bbeta ^ *)}}}
            	\exp{\brac[\Bigg]{-\frac{t_2(1-\delta)\bigl\Vert(\bI-\bP_{\mathbb{G}(\bbeta)})(\bX,\bZ)\btheta^*\bigr\Vert_2^2}{\sigma^2}}}\\
            	\leq&\exp{\brac[\Bigg]{2t_2\abs{\mathbb{G}(\bbeta)\backslash\mathbb{G}(\bbeta ^ *)}-\frac{t_2(1-\delta)}{\sigma^2} n d(\bbeta,\bbeta ^ *) c_{\min}}}
            \end{align*}
            Combining the two terms and set $t_1=\frac{1}{5}, t_2=\frac{3}{8}, \delta=\frac{4}{5}$, we have that Equation~(\ref{Equation:suff:prob1}) is less than or equal to:
            \begin{align*}
            	&\exp\brac[\bigg]{\frac{2t_1^2-t_1\delta}{\sigma^2}n d(\bbeta,\bbeta ^ *) c_{\min}}+
            	\exp{\brac[\Bigg]{2t_2\abs{\mathbb{G}(\bbeta)\backslash\mathbb{G}(\bbeta ^ *)}-\frac{t_2(1-\delta)}{\sigma^2} n d(\bbeta,\bbeta ^ *) c_{\min}}},\\
            	=&2\exp\brac[\bigg]{-\frac{3}{40}\frac{n d(\bbeta,\bbeta ^ *)c_{\min}}{\sigma^2}+\frac{3}{4}\abs{\mathbb{G}(\bbeta)\backslash\mathbb{G}(\bbeta ^ *)}}.
            \end{align*}
            Finally, we bound the probability that the $L_0$ estimator fails to specify the true grouping: 
	    	\begin{align}
    			&\mbP(\hat\btheta\neq \hat\btheta^{\mathrm{ol}})\label{Equation:suff:prob2}\\
    			\leq &\sum_{\omega\in\{\mathbb{G}(\bbeta)|\btheta\in\bTheta(K_0,s_0),\mathbb{G}(\bbeta)\neq\mathbb{G}(\bbeta ^ *)\}}\mbP \paren[\bigg]{\min_{\substack{\btheta\in\bTheta(K_0,s_0)\\ \mathbb{G}(\bbeta)=\omega}}\norm{ \bY-(\bX,\bZ)\btheta}_2^2<\norm{ \bY-(\bX,\bZ)\hat\btheta^{\mathrm{ol}}}_2^2},\nonumber\\
    			\leq &\sum_{i=1}^{s_0}\sum_{j=0}^i {{s_0}\choose{i}}K_0^{i} {{p-s_0}\choose{j}}K_0^j\,2\exp\paren[\bigg]{-\frac{3n i c_{\min}}{40\sigma^2}+\frac{3}{4}(2i+j)},\nonumber\\  
    			\leq &\sum_{i=1}^{s_0}\sum_{j=0}^i (K_0 s_0)^{i} \paren[\big]{K_0(p-s_0)}^j\,2\exp\paren[\bigg]{-\frac{3n i c_{\min}}{40\sigma^2}+\frac{3}{4}(2i+j)},\nonumber\\  
    			= &\sum_{i=1}^{s_0}  \,2\exp\paren[\bigg]{-\frac{3n i c_{\min}}{40\sigma^2}+\frac{3}{2}i+i\log(K_0s_0)}\sum_{j=0}^i  \exp\brak[\Big]{j\brac[\big]{\frac{3}{4}+\log\paren[\big]{K_0(p-s_0)}}},\nonumber\\ 
    			\leq &\frac{2}{1-e^{-3/4}}\sum_{i=1}^{s_0} \,\exp\brac[\bigg]{-\frac{3n i c_{\min}}{40\sigma^2}+\frac{3}{2}i+i\log(K_0s_0)+\frac{3}{4}i+i\log\paren[\big]{K_0(p-s_0)}}.\nonumber
			\end{align}
			Due to $\log\paren[\big]{K_0(p-s_0)}+\log(K_0s_0)\leq \log\paren[\big]{\frac{K_0^2 p^2}{4}}\leq 2\paren[\big]{\log(K_0 p)-\log(2)}$, we have Equation~(\ref{Equation:suff:prob2}) less than or equal to: 
			\[
			\begin{aligned}
			    \frac{2}{1-e^{-3/4}}\sum_{i=1}^{s_0} \,\exp\brac[\bigg]{-\frac{3n i c_{\min}}{40\sigma^2}+\frac{9}{10}i+2i\log\paren[\big]{K_0 p}}.
		    \end{aligned}
		    \]
		When $c_{\min}\geq \frac{\sigma^2}{n}\paren[\Big]{27\log(K_0 p)+12}$, we have Equation~(\ref{Equation:suff:prob2}) less than or equal to: 
		\[
			\begin{aligned}
			    \frac{2}{1-e^{-3/4}} \frac{\exp\brac[\bigg]{-\frac{3n}{40\sigma^2}\paren[\big]{c_{\min}-\frac{80}{3}\sigma^2\log(K_0 p)/n-12\sigma^2/n}}}{1-\exp\brac[\bigg]{-\frac{3n}{40\sigma^2}\paren[\big]{c_{\min}-\frac{80}{3}\sigma^2\log(K_0 p)/n-12\sigma^2/n}}}.
		    \end{aligned}
		    \]
		Due to the fact that $\mbP(\hat\btheta\neq \hat\btheta^{\mathrm{ol}})\leq 1$ and that $\frac{2}{1-e^{-3/4}}\frac{x}{1-x}\leq 6x$ when $0\leq\frac{2}{1-e^{-3/4}}\frac{x}{1-x}\leq 1$ and $0<x<1$, we have Equation~(\ref{Equation:suff:prob2}) less than or equal to: 
		\[
		 6\exp\brak[\bigg]{-\frac{3n}{40\sigma^2}\brac[\bigg]{c_{\min}-\frac{\sigma^2}{n}  \paren[\Big]{27\log(pK_0)+12}}}.\] 
	\end{proof}

\subsubsection{Proof of Corollary~\ref{cor:high-d}}

\begin{proof}
    Note that 
    $
       \{ \hat{\btheta}^{\mathrm{sg}} = \hat{\btheta}^{\mathrm{ol}}  \} = \{ \hat{\btheta}^{\mathrm{g}} = \hat{\btheta}^{\mathrm{ol}} \}\cap \cE .
    $
    Then 
    \[
    \begin{aligned}
        \mbP (\hat{\btheta}^{\mathrm{sg}} = \hat{\btheta}^{\mathrm{ol}}) & = \PP\Bigl(  \{ \hat{\btheta}^{\mathrm{g}} = \hat{\btheta}^{\mathrm{ol}} \}\cap \cE   \Bigr) \\
        & = 1 - \PP \Bigl(  \{ \hat{\btheta}^{\mathrm{g}} \ne \hat{\btheta}^{\mathrm{ol}} \} \cup \cE^c        \Bigr) \\
        & \ge 1 - \PP(\hat{\btheta}^{\mathrm{g}} \ne \hat{\btheta}^{\mathrm{ol}}) - \PP(\cE^c).
    \end{aligned}
    \]
    The conclusion immediately follows by combining this with Theorem \ref{thm:suff}.
\end{proof}

\subsubsection{Proof of Theorem~\ref{the:nece}} 
\label{app:2}
	\begin{proof} 
	    Consider a measurable space $(\mathcal{X},\mathcal{A})$ and a measurable function class $\Psi_t := \{\psi: \cX \to [t]\}$. By Lemma 2.7 in \cite{birge1983approximation}: for any sequence of $t\geq 2$ probability distributions $\PP_1,\hdots,\PP_t$ on the same measurable space $(\mathcal{X},\mathcal{A})$, we have that 
		\[
			\inf_{\psi\in\Psi_t}\sup_{j=1,\hdots,t} \PP_j(\psi(x)\neq j)\geq 1-\frac{t^{-2}\sum_{1\leq j,k\leq t}\mathrm{KL}(\PP_j,\PP_k)+\log{2}}{\log(t-1)},
		\] 
		where $\mathrm{KL}(\PP_j,\PP_k)$ is the Kullback-Leibler information for distributions $P_j$ versus $P_k$.\\ 
		For any $\gamma_{\min} > 0$, we can construct a collection of parameters of distinct groupings $\cS_{\gamma_{\min}} :=\{\bbeta^{(j)}\}_{j=0}^{\floor{\frac{K_0+3}{4}}p}\subseteq \bTheta(K_0,s_0)$ satisfying that  
		\begin{enumerate}[label={\roman*.}]
		    \item each entry of $\bbeta^{(j)}$'s belongs to $\cV := \{-\floor{\frac{K_0}{2}}\frac{\gamma_{\min}}{K_0},\hdots,-\frac{\gamma_{\min}}{K_0},0,\frac{\gamma_{\min}}{K_0},\hdots,\floor{\frac{K_0+1}{2}}\frac{\gamma_{\min}}{K_0}\}$; 
	        \item for any $0< j \leq \floor{\frac{K_0+3}{4}}p$, we have $\|{\bbeta^{(j)}-\bbeta^{(0)}}\|_0 \leq 2$;
	        \item for any $0< j \leq \floor{\frac{K_0+3}{4}}p$, we have 
	        $\|{\bbeta^{(j)}-\bbeta^{(0)}}\|_{1} = \frac{1}{K_0}\gamma_{\min} \text{ or } \frac{2}{K_0}\gamma_{\min}$;   
    	    \item for any $0< j \leq \floor{\frac{K_0+3}{4}}p$, we have 
	        $\|{\bbeta^{(j)}-\bbeta^{(0)}}\|_{\infty} = \frac{1}{K_0}\gamma_{\min} \text{ or } \frac{2}{K_0}\gamma_{\min}$.   
		\end{enumerate}
		
		Below we give the detailed construction of $\cS_{\gamma_{\min}}\subseteq\{\bbeta^{(0)}\}\cup\widetilde\cS_{\gamma_{\min}}\cup\overline\cS_{\gamma_{\min}}\cup\check\cS_{\gamma_{\min}}\cup\{{\bbeta}^{(01)},{\bbeta}^{(02)}\}$ where ${\bbeta}^{(01)}$ and ${\bbeta}^{(02)}$ are two variants of $\bbeta^{(0)}$ defined blow. 
		\begin{enumerate}
		    \item Set $\bbeta^{(0)}$ as any parameter with components valued in $\cV$ such that $|\cG(\bbeta^{(0)};0)|=p-K_0+1$, $|\cG(\bbeta^{(0)};\frac{\gamma_{\min}}{K_0})|=0$ and each of the rest $K_0-1$ groups has only one covariate, as shown in Figure~\ref{Figure:ConstructionBeta0}. \label{Step:ConstructionBeta0}
		    
		    \item Consider multiset $\widetilde\cS_{\gamma_{\min}} := \{\tilde\bbeta ^ {(j, k)}\}_{j \in [p-K_0+1], k \in \brak[\big]{\floor{\frac{K_0+1}{2}}}}$. Each $\tilde\bbeta ^ {(j, 1)}$ is generated by modifying $\bbeta^{(0)}$ via moving the $j$th covariate of $0$-group in $\bbeta ^ {(0)}$, i.e., $\cG(\bbeta ^ {(0)}; 0)$, to $\frac{\gamma_{\min}}{K_0}$-group. Each $\tilde\bbeta^{(j,k)}$ with $k>1$ is created by modifying $\tilde\bbeta^{(j,1)}$ through moving the covariate in $\frac{k}{K_0}\gamma_{\min}$-group to $\frac{k-1}{K_0}\gamma_{\min}$-group. (see Figure~\ref{Figure:ConstructionStep2}) \label{Step:Construction2}
		    
		    \item When $K_0\geq 5$, consider multiset $\overline\cS_{\gamma_{\min}} := \{\bar\bbeta ^ {(j, k)}\}_{j\in\brak[\big]{\floor{\frac{K_0+1}{2}}-2},k\in\brak[\big]{\floor{\frac{K_0}{2}}+1}}$. Each $\bar\bbeta ^ {(j, 1)}$ is generated by modifying $\bbeta^{(0)}$ via moving the covariate in $\frac{j+2}{K_0}\gamma_{\min}$-group to $\frac{j+1}{K_0}\gamma_{\min}$-group. Each $\bar\bbeta^{(j,k)}$ with $k>1$ is created by modifying $\bar\bbeta^{(j,1)}$ through moving the covariate in $-\frac{k}{K_0}\gamma_{\min}$-group to $-\frac{k-1}{K_0}\gamma_{\min}$-group. (see Figure~\ref{Figure:ConstructionStep3}) \label{Step:Construction3}
		    
		    \item When $K_0\geq 2$, consider multiset $\check\cS_{\gamma_{\min}} := \{\check\bbeta ^ {(j)}\}_{j\in\brak[\big]{\floor{\frac{K_0}{2}}}}$. Each $\check\bbeta ^ {(j)}$ is generated by modifying $\bbeta^{(0)}$ via moving the covariate in $-\frac{j}{K_0}\gamma_{\min}$-group to $-\frac{j-1}{K_0}\gamma_{\min}$-group. (see Figure~\ref{Figure:ConstructionStep4}) \label{Step:Construction4}
		    
		    \item When $K_0\geq 2$, construct ${\bbeta}^{(01)}$ by modifying $\bbeta^{(0)}$ via moving one covariate in $0$-group to $-\frac{1}{K_0}\gamma_{\min}$-group. When $K_0\geq 4$, construct ${\bbeta}^{(02)}$  by modifying $\bbeta^{(0)}$ via moving one covariate in $0$-group to $-\frac{2}{K_0}\gamma_{\min}$-group. (see Figure~\ref{Figure:ConstructionStep5}) \label{Step:Construction5}

		   % \item Consider multiset $\widetilde\cS_{\gamma_{\min}} := \{\bbeta ^ {(j, k)}\}_{j \in [p], k \in [K]}$, where each $\bbeta ^ {(j, k)}$ is generated by moving covariate $j$ from its group in $\bbeta ^ {(0)}$, i.e., $\cG(\bbeta ^ {(0)}; \beta ^ {(0)}_j)$, to a different group (see Figure~\ref{Figure:ConstructionStep2}). Note that there are identical elements in $\widetilde \cS_{\gamma_{\min}}$. \label{Step:Construction2}
		   % 
		   % \item To modify $\widetilde \cS_{\gamma_{\min}}$ to be a set without repetitive elements, consider the three types of moves in Step~\ref{Step:Construction2} that generate identical groupings:
		   %     \begin{enumerate}[label={\Roman*.}]
	    %        	\item Moving a covariate from a singleton non-zero group to another non-zero singleton group (See Figure~\ref{Figure:ConstructionDuplicateI});  
	    %        	\item Moving a covariate from a non-zero group of size 2 to an empty non-zero group (See Figure~\ref{Figure:ConstructionDuplicateII});  
	    %        	\item Moving a covariate from a singleton non-zero group to an empty non-zero group (See Figure~\ref{Figure:ConstructionDuplicateIII}).  
	    %    	\end{enumerate}
		   %     We can avoid such duplication by further moves represented by the dotted arrows in Figures~\ref{Figure:ConstructionDuplicateI}, \ref{Figure:ConstructionDuplicateII} and \ref{Figure:ConstructionDuplicateIII}.
		   %     \label{Step:ConstructionDuplication}
		\end{enumerate}
		
	    \input{Figures.tex}
	    
		Note that the constructed $\bbeta$'s in Steps~\ref{Step:ConstructionBeta0}-\ref{Step:Construction5} are distinct and with cardinality at least $\floor[\big]{\frac{K_0+3}{4}}p+1$.\\
		
		Then for any $0\leq j < j' \leq \floor[\big]{\frac{K_0+3}{4}}p$, we have
	    \[
			\begin{aligned}
				\mathrm{KL} & \{\cN(\bX\bbeta^{(j)}, \sigma^2 \bI_n), \cN(\bX\bbeta^{(j')},\sigma^2 \bI_n)\} =\frac{1}{2\sigma^2}\bigl\Vert \bX(\bbeta^{(j)}-\bbeta^{(j')})\bigr\Vert_2^2 \\
				& \leq\frac{2\max_{1\leq j\leq p}\left\Vert \bX_j\right\Vert_2^2\gamma_{\min}^2}{\sigma^2 K_0^2} = 
				\frac{2n\gamma^2_{\min} r(\bX, \bZ, K_0, s_0)}{\sigma ^ 2 K_0^2}\min_{\substack{{\btheta \in \bTheta(K_0, s_0)}\\ {|\beta_j-\beta_{j'}|\geq 1, \forall \beta_j \neq \beta_{j'} }}}c_{\min}(\btheta, \bX, \bZ) \\
				& = \frac{2nr(\bX, \bZ, K_0, s_0)}{\sigma ^ 2}\min_{\substack{{\btheta \in \bTheta(K_0, s_0)}\\ {|\beta_j-\beta_{j'}|\geq \gamma_{\min}/K_0, \forall \beta_j \neq \beta_{j'} }}}c_{\min}(\btheta, \bX, \bZ) \\
				& \leq \frac{2nr(\bX, \bZ, K_0, s_0)}{\sigma ^ 2}\min_{\btheta:\bbeta\in\cS_{\gamma_{\min}}}c_{\min}(\btheta, \bX, \bZ). 
				\end{aligned}
		\]
%		where $r(\bX,Z,s)=\frac{\max_{1\leq j\leq P}\frac{1}{N}\left\Vert \bX_j\right\Vert_2^2}{\min_{\substack{\btheta\in\Theta(K,s)\\|\bbeta_{j}|\geq 1:j\in\mathcal{G}^c(\bbeta,0)}}c_{\min}(\btheta,\bX,Z,s)}$.\\
		For any estimator $\hat\bbeta$ of $\bbeta ^ *$, we can define $\hat\bbeta^\dagger=\left\{
		\begin{aligned}
		    &\hat\bbeta, &\text{when }\hat\bbeta\in S_{\gamma_{\min}},\\
		    &\text{uniform}(S_{\gamma_{\min}}), &\text{otherwise}.
		\end{aligned}
		\right.$ \\ Then we can apply Lemma 2.7 in \cite{birge1983approximation} to probability distributions $\{\cN(\bX\bbeta,\sigma^2 \bI_n)\,|\,\bbeta\in \cS_{\gamma_{\min}}\}$. It follows that
		\begin{align*}
		    &\inf_{\hat\bbeta}\sup_{\btheta^*\in \bTheta_c(K_0, s_0, \min_{{\btheta:\bbeta\in \cS_{\gamma_{\min}}}}c_{\min}(\btheta, \bX, \bZ))}\PP(\mathbb{G}(\hat\bbeta)\neq\mathbb{G}(\bbeta ^ *))\\
		    \geq &\inf_{\hat\bbeta}\sup_{\btheta^*:\bbeta\in\cS_{\gamma_{\min}} }\PP(\mathbb{G}({\hat\bbeta})\neq\mathbb{G}(\bbeta ^ *)) \ge \inf_{\hat\bbeta}\sup_{\btheta^*:\bbeta\in\cS_{\gamma_{\min}} }\PP(\mathbb{G}({\hat\bbeta ^ {\dagger}})\neq\mathbb{G}(\bbeta ^ *))\\
		    = & \inf_{\psi \in \Psi_{\floor{\frac{K_0+3}{4}}p+1}}\sup_{\btheta^*:\bbeta\in\cS_{\gamma_{\min}} }\PP(\mathbb{G}(\bbeta ^ {(\psi)})\neq\mathbb{G}(\bbeta ^ *)) \\
		    \geq & 1-\frac{2nr(\bX,\bZ,K_0,s_0)\min_{{\btheta:\bbeta\in \cS_{\gamma_{\min}}}}c_{\min}(\btheta, \bX, \bZ)+\sigma^2\log2}{\sigma^2 \log(\floor{\frac{K_0+3}{4}}p)}.
		\end{align*}
        When $\gamma_{\min}$ varies from $0$ to $\infty$, $\min_{\btheta:\bbeta\in \cS_{\gamma_{\min}}}c_{\min}(\btheta,\bX,\bZ)$ varies from $0$ to $\infty$. Then for any $\ell>0$ we have
        \begin{align*}
			\inf_{\hat\bbeta}\sup_{\btheta^*\in \bTheta_c(K_0, s_0, \ell)} \mbP\paren[\big]{\mathbb{G}(\hat\bbeta)\neq \mathbb{G}(\bbeta ^ *)}\geq 1-\frac{2nr(\bX,\bZ,s)\ell+\sigma^2\log 2}{\sigma^2\log(\floor{\frac{K_0+3}{4}}p)}.
		\end{align*}
		Then $\inf_{\hat\bbeta}\sup_{\btheta^*\in \bTheta_c(K_0,s_0,\ell)}\mbP\paren[\big]{\mathbb{G}(\hat\bbeta)\neq\mathbb{G}(\bbeta ^ *)}\to 0, \text{ as }n,p\to\infty$ implies 
		\[
		l\geq \frac{\sigma^2\paren[\big]{\log(\floor{\frac{K_0+3}{4}}p)-\log 2}}{2nr(\bX,\bZ,K_0,s_0)}. 
		\]
		\end{proof}

\subsection{Proof of Proposition \ref{prop:beta}}
\begin{proof}
    Consider $s_0=p$ and $K_0=2$. We have $\GG(\bbeta^*) = \{\cG_1^*,\cG_2^*\}$, where $\cG_1^* = \{j:\beta_j^* = \gamma^*_1\}$ and $\cG_2^* = \{j:\beta_j^* = \gamma^*_2\}$ with $|\gamma^*_1-\gamma^*_2| \ge 1$. Write $k^*_1 = |\cG_1^*|$ and $k^*_2 = |\cG_2^*|$. For any $\bbeta$ such that $\GG(\bbeta) = \{\cG_1,\cG_2\}\ne \GG(\bbeta^*)$, let $k_1=|\cG_1|$ and $k_2=|\cG_2|$. We now consider two possible grouping structures.
    
\noindent{Case (i): Without loss of generality, $\cG_1\subset\cG_1^* $ and $\cG_2^*\subset\cG_2$}. In this case, we have 
\begin{eqnarray*} 
        d(\bbeta,\bbeta^*)  =  |\cG_1^*\setminus\cG_1| & = & k_1^*-k_1, \\
       \argmin_{\bbeta'_{\cG_1} \text{ s.t. } \GG(\bbeta')=\GG(\bbeta)} \|\bbeta'_{\cG_1}-\bbeta^*_{\cG_1}\|_2^2 & = & \gamma_1^* \mathbf{1}_{|\cG_1|}, \\
       \argmin_{\bbeta'_{\cG_2} \text{ s.t. } \GG(\bbeta')=\GG(\bbeta)} \|\bbeta'_{\cG_2}-\bbeta^*_{\cG_2}\|_2^2 & = & \frac{(k_1^*-k_1)\gamma^*_1 + k_2^*\gamma^*_2}{k_2} \mathbf{1}_{|\cG_2|}.
    \end{eqnarray*} 
    Note that $k_1+k_2=s_0$ and $k_1^*+k_2^*=s_0$. Accordingly,
    \[
    \begin{aligned}
              & \min_{\bbeta' \text{ s.t. } \GG(\bbeta')=\GG(\bbeta) }\|\bbeta' - \bbeta^*\|_2^2 \\ 
              =~& (k_1^*-k_1)\biggl(\gamma_1^* - \frac{(k_1^*-k_1)\gamma^*_1 + k_2^*\gamma^*_2}{k_2} \biggr)^2 + k^*_2 \biggl(\gamma_2^* - \frac{(k_1^*-k_1)\gamma^*_1 + k_2^*\gamma^*_2}{k_2}\biggr)^2 \\
              = ~ & \frac{(k_1^*-k_1) (s_0 - k_1^*)^2(\gamma_1^*-\gamma_2^*)^2 }{(s_0 - k_1)^2} + \frac{(s_0-k_1^*)(k_1^*-k_1)^2(\gamma_1^*-\gamma_2^*)^2 }{(s_0-k_1)^2} \\
              = ~ &\frac{(k_1^*-k_1)(s_0-k_1^*)(\gamma_1^*-\gamma_2^*)^2}{s_0-k_1}. 
    \end{aligned}
    \]
    Since $C_l\le k_1^*/k_2^*\le C_h$ for some universal constants $C_h>C_l>0$, we always have that $k^*_1 = c^*_1s_0$ and $k^*_2 = (1-c_1^*)s_0$ for some constant $0<c_1^*<1$. 
    Then
    \[
    \begin{aligned}
               & \frac{\min_{\bbeta' \text{ s.t. } \GG(\bbeta')=\GG(\bbeta) }\|\bbeta' - \bbeta^*\|_2^2}{d(\bbeta,\bbeta^*)} \\
               = ~ & \frac{(s_0-k_1^*)(\gamma_1^*-\gamma_2^*)^2}{s_0-k_1} \ge \frac{(1-c_1^*)s_0(\gamma_1^*-\gamma_2^*)^2}{s_0} = (1-c_1^*)(\gamma_1^*- \gamma_2^*)^2 \ge 1-c_1^*.
    \end{aligned}
    \]
    
\noindent{Case (ii): 
    $\cG_1\cap\cG_1^*\ne \emptyset$ and $\cG_1\cap\cG_2^*\ne\emptyset$}. Let $d = |\cG_1^*\setminus \cG_1|$. Then 
    \[
       d(\bbeta,\bbeta^*) = \min\{ k_1-k_1^*+2d, s_0-(k_1-k_1^*+2d).  \}
    \]
   Also,
    \[
    \begin{aligned}
       & \argmin_{\bbeta'_{\cG_1} \text{ s.t. } \GG(\bbeta')=\GG(\bbeta)} \|\bbeta'_{\cG_1}-\bbeta^*_{\cG_1}\|_2^2 = \frac{(k_1^*-d)\gamma^*_1 + (k_1-k_1^*+d)\gamma^*_2}{k_1} \mathbf{1}_{|\cG_1|}, \\
       & \argmin_{\bbeta'_{\cG_2} \text{ s.t. } \GG(\bbeta')=\GG(\bbeta)} \|\bbeta'_{\cG_2}-\bbeta^*_{\cG_2}\|_2^2 = \frac{d\gamma^*_1+(k_2-d)\gamma^*_2}{k_2} \mathbf{1}_{|\cG_2|}.
    \end{aligned}
    \]
    Accordingly we have
    \[
    \begin{aligned}
           & \min_{\bbeta' \text{ s.t. } \GG(\bbeta')=\GG(\bbeta) }\|\bbeta' - \bbeta^*\|_2^2 \\ 
           = ~ & (k_1^*-d)\biggl( \gamma_1^* -\frac{(k_1^*-d)\gamma^*_1 + (k_1-k_1^*+d)\gamma^*_2}{k_1}\biggr)^2 \\
           & + (k_1-k_1^*+d)\biggl( \gamma^*_2 -  \frac{(k_1^*-d)\gamma^*_1 + (k_1-k_1^*+d)\gamma^*_2}{k_1}  \Bigr)^2 \\
           & + d\Bigl(\gamma_1^* - \frac{d\gamma^*_1+(k_2-d)\gamma^*_2}{k_2}\Bigr)^2 + (k_2-d)\Bigl( \gamma_2^* - \frac{d\gamma^*_1+(k_2-d)\gamma^*_2}{k_2}  \Bigr)^2 \\
           = ~ & \frac{(k_1^*-d)(k_1-k^*_1+d)(\gamma_1^*-\gamma_2^*)^2}{k_1} + \frac{d(k_2-d)(\gamma_1^*-\gamma_2^*)^2}{k_2}.
    \end{aligned}
    \]
    Similar to case (i), we have
    \begin{align}\label{equ:proof_betagap}
        & \frac{\min_{\bbeta' \text{ s.t. } \GG(\bbeta')=\GG(\bbeta) }\|\bbeta' - \bbeta^*\|_2^2}{d(\bbeta,\bbeta^*)} \\ 
        = ~& \frac{(k_1^*-d)(k_1-k^*_1+d)/k_1 + d(k_2-d)/k_2}{ \min\{ k_1-k_1^*+2d, s_0-(k_1-k_1^*+2d)  \} }(\gamma_1^*-\gamma_2^*)^2. \nonumber
    \end{align}
%    for some constant $c:= c(c_1^*)>0$. Combining cases (i) and (ii) concludes the proof.
     The numerator term in \eqref{equ:proof_betagap} equals to
     \[
     \begin{aligned}
        & k_1\biggl( \frac{k_1^*-d}{k_1} \biggr) \biggl(1- \frac{k_1^*-d}{k_1}\biggr) + k_2\biggl(\frac{d}{k_2}\biggr)\biggl(1-\frac{d}{k_2}\biggr) \\
        = ~ & k_1\biggl( \frac{c_1^*s_0-d}{k_1} \biggr) \biggl(1- \frac{c_1^*s_0-d}{k_1}\biggr) + k_2\biggl(\frac{d}{k_2}\biggr)\biggl(1-\frac{d}{k_2}\biggr).
     \end{aligned}
     \]
     The denominator term in \eqref{equ:proof_betagap} is always upper bounded by $s_0/2$. Note that $k_1-k_1^*+d>0$ and $d<k_1^*$. We consider two situations with different orders of $d$. 
     
     \noindent\text{(i)} Under $d = o(s_0)$, 
     we have
     \begin{equation}\label{equ:nume}
        k_1\biggl(\frac{k_1^*-d}{k_1}\biggr)\biggl(1-\frac{k_1^*-d}{k_1} \biggr) \approx c_1^*s_0 \frac{k_1-k_1^*+d}{k_1}.
     \end{equation}
     Now it is possible that $k_1-k_1^*+d = c_1s_0$ with some $c_1\in(0,1)$ or $k_1-k_1^*+d = o(s_0)$. If  $k_1-k_1^*+d = c_1s_0$ with some $c_1\in(0,1)$, we have $\eqref{equ:nume}\gtrsim s_0$ and consequently $\eqref{equ:proof_betagap}\gtrsim 1$.
     If $k_1-k_1^*+d = o(s_0)$, it holds that $\min\{k_1-k_1^*+2d, s_0-(k_1-k_1^*+2d)\} = k_1-k_1^* + 2d$ with sufficiently large $s_0$. We have
     \[
     \eqref{equ:proof_betagap} = \frac{(k_1^*-d)(k_1-k_1^*+d)/k_1}{k_1-k_1^*+2d} \gtrsim 1.
      \]
     
\noindent\text{(ii)}  
     Under $d\approx s_0$, by noticing that $d<k_2$, we have $k_2\approx s_0$ and
     \[
        k_2\biggl(\frac{d}{k_2}\biggr)\biggl(1-\frac{d}{k_2}\biggr)  = d \frac{k_2-d}{k_2} \approx k_2-d.
     \]
     If $k_2-d\approx s_0$, then $\eqref{equ:nume}\gtrsim s_0$ and $\eqref{equ:proof_betagap} \gtrsim 1$. If $k_2-d = o(s_0)$, we have $k_1-k_1^*+d = c_2^*s_0 - o(s_0)$, thus the first term of the numerator in  \eqref{equ:proof_betagap} is lower bounded by the order of $s_0$. 
    
    Combining all the cases concludes the proof.
\end{proof}

	\bibliographystyle{imsart-nameyear}
	\bibliography{ref2}
\end{document}